\documentclass{article}

\usepackage[utf8]{inputenc} 
\usepackage[normalem]{ulem}
\usepackage{microtype}
\usepackage{graphicx}
\usepackage{tikz}
\usepackage{booktabs}
\usepackage{fullpage,hyperref,graphicx,amsmath,amsfonts,subcaption,amssymb,bm,url,breakurl,epsfig,epsf,color,MnSymbol,fmtcount,semtrans,caption,multirow,comment}
\usepackage{graphicx}
\usepackage{adjustbox}
\usepackage{nicefrac}  
\usepackage{hyperref}

\usepackage{wrapfig}
\usepackage{amsmath,mathtools,amssymb,bm,breakurl,MnSymbol,fmtcount,semtrans,multirow,comment,boldline,algorithm,algorithmic}

\usepackage{tikz,wrapfig,graphicx,epsfig,epsf,caption,subcaption,pgfplots,movie15}
\usetikzlibrary{pgfplots.groupplots}

\usepackage{url}            
\usepackage{booktabs}       
\usepackage{amsfonts}       
\usepackage{nicefrac}       

\usepackage{hyperref,color}
\definecolor{darkred}{RGB}{150,0,0}
\definecolor{darkgreen}{RGB}{0,150,0}
\definecolor{darkblue}{RGB}{0,0,200}
\usepackage{dsfont}

\newtheorem{theorem}{Theorem}[section]

\newtheorem{lemma}[theorem]{Lemma}

\newtheorem{proposition}[theorem]{Proposition}
\newtheorem{definition}[theorem]{Definition}

\newcommand{\eps}{\varepsilon}

\newcommand{\act}[1]{\text{ReLU}(#1)}

\newcommand{\som}[1]{}
\newcommand{\ysatt}[1]{}

\newcommand{\distas}{\overset{\text{i.i.d.}}{\sim}}

\newcommand{\beq}{\begin{equation}}

\newcommand{\eeq}{\end{equation}}

\newcommand{\fn}{\text{final}}

\newcommand{\nn}{\nonumber}
\newcommand{\la}{\lambda}

\newcommand{\bt}{\boldsymbol{\theta}}
\newcommand{\bg}{\boldsymbol{\gamma}}
\newcommand{\bz}{\boldsymbol{\zeta}}
\newcommand{\btr}{\boldsymbol{\theta}^R}
\newcommand{\bth}{\hat{\boldsymbol{\theta}}}

\newcommand{\btb}{\bar{\boldsymbol{\theta}}}
\newcommand{\bta}{\boldsymbol{\theta}_{\text{aux}}}
\newcommand{\bOt}{\bigotimes}

\newcommand{\Lcp}{{\cal{L}}_{\text{PO}}}
\newcommand{\Lca}{{\cal{L}}_{\text{AO}}}

\newcommand{\V}{{\mtx{V}}}

\newcommand{\ini}{\text{init}}

\newcommand{\Lc}{{\cal{L}}}
\newcommand{\Hc}{{\cal{H}}}
\newcommand{\Lch}{{\hat{\cal{L}}}}

\newcommand{\ix}{\Delta}
\newcommand{\ixi}{{\Delta_{\W}}}
\newcommand{\ixo}{{\Delta_{\Vb}}}
\newcommand{\ixf}{{\Delta_1}}
\newcommand{\iii}{{\Delta_i}}
\newcommand{\iib}{{\bar{\Delta}_i}}
\newcommand{\ixs}{{\Delta_2}}
\newcommand{\Dc}{{\cal{D}}}

\newcommand{\Vc}{{\cal{V}}}

\newcommand{\La}{{\boldsymbol{\Lambda}}}

\newcommand{\bSi}{{\boldsymbol{{\Sigma}}}}

\newcommand{\bmu}{{\boldsymbol{{\mu}}}}

\newcommand{\onebb}{{\mathds{1}}}
\newcommand{\Iden}{{\mtx{I}}}
\newcommand{\M}{{\mtx{M}}}

\newcommand{\z}{{\vct{z}}}

\newcommand{\tn}[1]{\|{#1}\|_{\ell_2}}
\newcommand{\tone}[1]{\|{#1}\|_{\ell_1}}
\newcommand{\tin}[1]{\|{#1}\|_{\ell_\infty}}

\newcommand{\tf}[1]{\|{#1}\|_{F}}

\newcommand{\Cc}{\mathcal{C}}

\newcommand{\bteta}{\boldsymbol{\theta}}

\newcommand{\Sc}{\mathcal{S}}
\newcommand{\Mc}{{\cal{M}}}
\newcommand{\pa}{{\partial}}
\newcommand{\Nn}{\mathcal{N}}

\newcommand{\vb}{\vct{v}}

\newcommand{\Ic}{{\mathcal{I}}}
\newcommand{\Icn}{{\mathcal{I}}^N}

\newcommand{\Xb}{\mtx{\bar{X}}}

\newcommand{\w}{\vct{w}}

\newcommand{\li}{\left<}
\newcommand{\ri}{\right>}

\newcommand{\ab}{\vct{a}}
\newcommand{\bb}{\vct{b}}

\newcommand{\h}{\vct{h}}

\newcommand{\g}{{\vct{g}}}

\newcommand{\Fc}{\mathcal{F}}

\newcommand{\st}{\star}

\newcommand{\x}{\vct{x}}

\newcommand{\y}{\vct{y}}

\newcommand{\W}{\mtx{W}}
\newcommand{\bgl}{{~\big |~}}

\newcommand{\R}{\mathds{R}}
\newcommand{\Pro}{\mathds{P}}

\newcommand{\E}{\operatorname{\mathds{E}}}

\newcommand{\grad}[1]{{\nabla\Lc(#1)}}

\newcommand{\eb}{\vct{e}}

\newcommand{\vct}[1]{\bm{#1}}
\newcommand{\mtx}[1]{\bm{#1}}

\newcommand{\Pc}{{\cal{P}}}
\newcommand{\X}{{\mtx{X}}}

\newcommand{\Vb}{{\mtx{V}}}

\numberwithin{equation}{section} 

\def \endprf{\hfill {\vrule height6pt width6pt depth0pt}\medskip}
\newenvironment{proof}{\noindent {\bf Proof} }{\endprf\par}

\hypersetup{colorlinks=true, linkcolor=darkred, citecolor=darkgreen, urlcolor=darkblue}

\usepackage[utf8]{inputenc} 
\usepackage[T1]{fontenc}    
\usepackage{url}            
\usepackage{booktabs}       
\usepackage{microtype}      

\newcommand{\vs}{\vspace*{0pt}}

\newcommand{\mc}[1]{}

\title{Exploring Weight Importance and Hessian Bias in Model Pruning}
\date{}

\author{%
  Mingchen Li\thanks{Email: \texttt{mli176@ucr.edu}.~~~~~Computer Science and Engineering, University of California, Riverside.} \quad\quad Yahya Sattar\thanks{Email: \texttt{ysatt001@ucr.edu}.~~Electrical and Computer Engineering, University of California, Riverside.} \quad\quad Christos Thrampoulidis\thanks{Email: \texttt{cthrampo@ucsb.edu}.~Electrical and Computer Engineering, University of California, Santa Barbara.}\quad\quad Samet Oymak\thanks{Email: \texttt{oymak@ece.ucr.edu}.~Electrical and Computer Engineering, University of California, Riverside.}
}

\begin{document}

\maketitle

\vs\begin{abstract}

Model pruning is an essential procedure for building compact and computationally-efficient machine learning models. A key feature of a good pruning algorithm is that it accurately quantifies the relative importance of the model weights. While model pruning has a rich history, we still don't have a full grasp of the pruning mechanics even for relatively simple problems involving linear models or shallow neural nets. In this work, we provide a principled exploration of pruning by building on a natural notion of importance. For linear models, we show that this notion of importance is captured by covariance scaling which connects to the well-known Hessian-based pruning. We then derive asymptotic formulas that allow us to precisely compare the performance of different pruning methods. For neural networks, we demonstrate that the importance can be at odds with larger magnitudes and proper initialization is critical for magnitude-based pruning. Specifically, we identify settings in which weights become more important despite becoming smaller, which in turn leads to a catastrophic failure of magnitude-based pruning. Our results also elucidate that implicit regularization in the form of Hessian structure has a catalytic role in identifying the important weights, which dictate the pruning performance.

\end{abstract}

\vs\vs

\section{Introduction}\vs
\som{Consider kicking this out!!!}Contemporary machine learning models such as deep neural networks often achieve good statistical accuracy at the expanse of large model sizes. On the other hand, a growing list of application domains demand compact and energy efficient machine learning models. Model pruning (i.e.~sparsification) techniques are critical for addressing the challenge of building models that are simultaneously accurate and efficient. In this work, we investigate the fundamental principles of model pruning by exploring optimization dynamics and high-dimensional behavior of pruning approaches.
\vs Pruning methods have a rich history and the literature on neural network pruning goes back to 1980's \cite{mozer1989skeletonization,lecun1990optimal,hassibi1993second}. A fundamental approach in pruning is the accurate quantification of importance of each weight (i.e.~connections) so that when a weight is removed, we can know how much the model will suffer. An intuitive approach is pruning by the weight magnitude, i.e.~removing the weights below a certain threshold. A more principled approach is developing an importance (i.e.~saliency) criteria which captures the sensitivity of the loss with respect to the weights. For instance, Optimal Brain Damage (OBD) \cite{lecun1990optimal} and Optimal Brain Surgeon \cite{hassibi1993second,hassibi1994optimal} calculate a Hessian-based importance criteria by adjusting the magnitudes. Despite its practical significance, a statistical understanding of pruning presents interesting challenges. Deep networks are often trained in an over-parameterized regime where the network size is well beyond what is necessary for achieving zero training error. Thus, network weights can interpolate the data in many ways and it is not immediately clear which weight gets the credit for learning. Pruning typically happens after training this large initial network possibly without any $\ell_1,\ell_2$ regularization. Deep nets may also converge to different solutions under different initialization or data preprocessing. These motivate a careful study of pruning mechanics: Which approach works when? What is the role of initialization? Does over-parameterization affect the outcome and if so, can it be quantified?
\vs\noindent{\bf{Contributions:}} In this work, we explore model pruning, importance quantification and the role of Hessian structure in the pruning performance. We study three different importance criteria and corresponding pruning methods:  (i) Hessian-based importance (HI) and pruning (HP), (ii) Magnitude-based importance (MI) and pruning (MP), and a third notion, which we call (iii) Natural importance (NI) and pruning (NP). For linear models and shallow neural-networks, we design a class of {\em{equivalent problems}} which enable us to assess the role of {\em{Hessian structure}} on the robustness and performance of different importance measures. Our specific contributions are as follows.

\vs\noindent$\bullet$ {\bf{Understanding covariance bias and pruning performance:}} For linear models, we introduce a class of problems where Hessian, which corresponds to the feature covariance matrix, is varied using diagonal scaling, while preserving target labels. We show that for over-determined problems HI and NI exhibit scaling invariance, whereas, MI is highly brittle. For over-parameterized problems, we show that scaling invariance no longer holds and the covariance/Hessian structure dictates the eventual pruning performance. We introduce analytical performance formulas, precisely capturing these phenomena, revealing that {\em{implicit bias {(as enforced by the Hessian structure)}  can boost HP while hurting MP}}. Our approach also allows us to quantify {\em{negative bias}} when principal covariance directions are mis-aligned with the important weights. To the best of our knowledge, this is the first work that provides {\em{exact analytical formulas for the performance of MP/HP}}.

\vs\noindent$\bullet$ {\bf{Understanding Hessian bias and the role of initialization:}} For two-layer ReLU networks, we tackle the following question: {\em{If both layers are very large and can interpolate the training data, who contributes more towards learning, who gets pruned eventually {\color{black}and at what cost}?}} We study these questions via a simple, yet insightful, network initialization model and show that the answers depend crucially on the Hessian structure which governs the training dynamics. Our empirical study reveals that: (i) HI is invariant to Hessian bias and (ii) as MI decreases, NI (which captures the training/test accuracy) increases. To explain this, we first show that magnitudes of the weights and magnitudes of their Hessians move in opposing directions and then establish a {\em{``larger Hessian wins more''}} theorem which accurately quantifies the relative contribution of different weight groups (e.g.~layers) during training in terms of the Hessian bias.

\vs\vs\vs
\subsection{Related work}\vs
\vs
Our work relates to the literature on neural net pruning, implicit regularization and over-parameterization. For analysis, we also use tools related to high-dimensional statistics \cite{thrampoulidis2015regularized,OymLAS,thrampoulidis2018precise,hastie2019surprises}. 

\vs
\noindent {\bf{Implicit bias and over-parameterization:}} Contemporary deep networks often contain many more parameters than the dataset size and there is a growing literature dedicated to understanding their optimization/generalization properties and how over-parameterization can act as a catalyst. A key observation is that gradient-based algorithms are implicitly guided by the problem structure towards certain favorable solutions \cite{arora2018optimization,neyshabur2014search}.  For linear models, implicit bias phenomena is studied for various loss functions and algorithms (e.g.~logistic loss converging to max-margin solution on separable data) \cite{ji2018risk,soudry2018implicit,nacson2019stochastic}. Recent works show that such results continue to hold for nonlinear problems \cite{gunasekar2017implicit,oymak2019overparameterized,azizan2018stochastic}. This line of works led to the more recent generalization/optimization guarantees for deep networks and their connections to random features \cite{du2019gradient,allen2019convergence,chizat2019lazy,belkin2018understand,belkin2019does,liang2018just,mei2019generalization}. A related line of work connects the benefits of over-parameterization to the double descent phenomena \cite{nakkiran2019deep,belkin2019two,belkin2019reconciling,hastie2019surprises}. 


\mc{\cite{molchanov2019importance} discusses weight import w.r.t. loss change. limitation of weight pruning. proposes pruning methods for gradient and hessian weighing} 
\vs\noindent {\bf{Neural network pruning:}} The large model sizes in deep learning led to a substantial interest in model pruning/quantization \cite{han2015deep,hassibi1993second,lecun1990optimal}. The network pruning literature is diverse and involves various architectural, algorithmic, and hardware considerations \cite{sze2017efficient,han2015learning}. Recent works \cite{han2015learning,frankle2019lottery,franklestabilizing} use magnitude-based pruning criteria and achieve stellar performance. Related to over-parameterizarion, lottery ticket hypothesis \cite{frankle2018lottery} shows that large neural networks contain a small subset of favorable weights (for pruning) which can achieve similar performance as the original network when trained from same initialization. \cite{zhou2019deconstructing,malach2020proving} demonstrates that these subsets may achieve good test performance even without any training. \cite{tian2019luck} theoretically connects lottery tickets to over-parameterization.
Various saliency-based approaches are proposed for neural net pruning \cite{hassibi1993second,hassibi1994optimal,lecun1990optimal,dong2017learning}. \cite{lee2018snip,wang2020picking} prune the network before training by the connection sensitivity or preserving the gradient flow. \cite{shunshi2019one} uses Jacobian-based pruning for recurrent networks. Furthermore, \cite{aghasi2017net,oymak2018learning,jin2016training} uses $\ell_1$ penalization for pruning and provides certain provable guarantees. 

\vs
The rest of the paper is organized as follows. Section~\ref{sec: setup} sets the notation and introduces definitions on importance and pruning. Section \ref{sec linear} studies pruning for linear models, characterizes covariance bias, and introduces analytical performance formulas. Section \ref{neural sec} explores pruning for neural network and introduces results on optimization and pruning dynamics and Section~\ref{sec:conc} provides a discussion. 

\vs\vs\vs\section{Problem Setup}\vs\vs\vs\label{sec: setup}

We first set the notation. For a vector $\vb$, we denote by $\tn{\vb}$ its Euclidean norm. $\odot$ returns the Hadamard (entrywise) product of two vectors. The $(i,j)$-th element of a matrix $\M$ is denoted by $\M_{i,j}$. The minimum singular value, spectral norm, and Frobenius norm of $\M$ is denoted by $\sigma_{\min}(\M),\|\M\|,\tf{\M}$ respectively. $\Iden_k$ is the identity matrix of size $k$. The set $\{1,\dots,p\}$ is denoted by $[p]$. Given $\ix\subset[p]$, $\bar{\ix}=[p]-\ix$ and $\bt_{\ix}$ denotes the vector obtained by setting the entries of $\bt$ over $\bar{\ix}$ to zero. $\onebb_{p}$ denotes the all ones vector in $\R^{p}$.

To proceed, we review definitions that will be discussed throughout. Our discussion will stem from the following definition which captures the impact of a set of weights on the loss function.
\vs\begin{definition}[Natural importance (NI)]\label{def: weight impo} 
Given a loss function $\Lc(\bteta)$, a reference vector $\btr$ and set of indices $\ix\subseteq[p]$, note that $\btr_{{\ix}}+\bt_{\bar{\ix}}$ is the vector obtained by replacing the entries of $\bt$ at indices $\ix$ by the corresponding entries of $\btr$. The NI of the weights of $\bt$ over $\ix$ with respect to (w.r.t)~$\Lc$ is defined as
\[
\Icn_{\ix}(\bt,\btr) = \Lc(\btr_{{\ix}}+\bt_{\bar{\ix}}) - \Lc(\bt).
\]
\end{definition}
\mc{I feel definition of $\bt_{\bar{\ix}}^R$ is not consistent with $\bt_\ix$ but they have similar form. If we use same definition of $\bt_\ix$, current $\bt_{\bar{\ix}}^R$ can be written as $\bt_\ix$+$\bt_{\bar{\ix}}^R$}
When $\btr=0$, we will use the notation $\Icn_\ix(\bt)$. NI quantifies the {\em{exact change in the loss and captures the problem-dependent nature of pruning}}.\som{motivate better?} The loss function in practice can be training (or test) loss or classification error. Here, the vector $\btr$ aims to quantify the relative benefit of the change of weights of $\bt$ with respect to a reference. For our purposes, we discuss two choices for the reference vector, which we call pruning and init-pruning, respectively.

\mc{regular can write as $\bt_\ix$ and Init-pruning is $\bt_\ix+\bt^R_{\bar{\ix}}$}
$\bullet$ {\bf{(Regular) Pruning:}} This is the standard pruning where the goal is to obtain a sparse model, thus the reference vector is $\btr=0$.\\
$\bullet$ {\bf{Init-Pruning:}} Deep network training is often initialized from nonzero weights $\bt_0$ such as random initialization or pre-trained weights. In this case, the contribution of different weights throughout the optimization can be assessed with respect to the point of initialization by choosing $\btr=\bt_0$. 

We remark that, our characterization of the weight importance is similar to the saliency criterion which is widely used in literature on model pruning/trimming~\cite{lecun1990optimal,lee2018snip,mozer1989skeletonization,sum1999kalman}. Besides Definition \ref{def: weight impo}, we also consider two other commonly-accepted importance criteria, which can be viewed as proxies for NI. To keep the discussion focused, the next two definitions only consider regular pruning i.e.~$\bt^R=0$.
\vs\begin{definition}[Magnitude- and Hessian-based Importance] Recall Def.~\ref{def: weight impo}. Suppose $\Lc$ is twice differentiable with Hessian $\Hc(\bt)=\nabla^2\Lc(\bt)$. The MI $\Ic^M_{\ix}(\bt)$ and HI $\Ic^H_{\ix}(\bt)$ are defined as 
\begin{align}
&\Ic^M_\ix(\bt)=\sum_{i\in \ix} \bt_i^2\quad\text{and}\quad\Ic^H_\ix(\bt)=\sum_{i\in \ix} \Hc(\bt)_{i,i}\bt_i^2.
\end{align}
\end{definition}
Observe that our definition of HI is based on Optimal Brain Damage (OBD) \cite{lecun1990optimal}. Next, we define pruning based on a given importance criteria. A pruning algorithm identifies a set of weights with the smallest importance and sets them to zero. 
\mc{I feel current def is counter intuitive. why don't we use $\Pi_s(\bt)=\bt_{\ix}\quad\text{where}\quad \ix=\arg\max_{|\ix|=s}\Ic_\ix(\bt).$}
\vs\begin{definition} [Pruning] Given an importance criteria $\Ic$ (e.g.~$\Ic^N,\Ic^M,\Ic^H$), a vector $\bt$, and a target sparsity $s$, the pruning algorithm returns an $s$-sparse model $\Pi_s(\bt)$ (e.g.~$\Pi^{N}_s$, $\Pi^{M}_s$, $\Pi^{H}_s$) where
\[
\Pi_s(\bt)=\bt_{\bar{\ix}}\,,\quad\text{for}\quad \ix=\arg\min_{|\ix|=p-s}\Ic_\ix(\bt). 
\]
\end{definition}
We will study and compare three different methods of pruning, namely, magnitude-based~(MP), Hessian-based~(HP) and natural pruning~(NP). While NI captures the ``true importance'', NP is a combinatorially challenging subset selection problem and HP and MP provides computationally-efficient alternatives. For MP, this definition reduces to the hard-thresholding operation. Furthermore, MP and HP coincide when the Hessian has equal diagonal entries. We will focus our attention on pruning the trained model. Thus, typically we are interested in pruning the minimizer of the empirical (or population) loss. The following sections will relate these pruning methods, compare their performances, and explore the role of implicit regularization in pruning.

\vs\vs
\section{Importance and Covariance Bias for Linear Models}\label{sec linear}\vs\vs



This section provides our results on pruning linear models and the role of feature covariance. Given a data distribution $\Dc$, we obtain a dataset $\Sc$ containing $n$ i.i.d.~samples $\Sc=(\x_i,y_i)_{i=1}^n\distas \Dc$. Let $(\x,y)\sim\Dc$ be a generic sample. We assume $(\x,y) \in (\R^p, \R)$ has finite second moments. 

\noindent {\bf{Covariance/Hessian structure:}} To understand the role of feature covariance (i.e.~Hessian) on pruning, we introduce a class of datasets where the input features are shaped by an invertible diagonal scaling matrix $\La\in\R^{p\times p}$ while output label $y$ is preserved. Here, a key motivation is modeling the properties of neural networks where the Hessian spectrum is not flat and often low-rank \cite{hastie2019surprises,sagun2017empirical,papyan2018full,mei2019generalization,arora2019fine}. The intuition is that the importance of a weight captures the contribution of the corresponding input feature and should be invariant to how the feature is scaled. Perhaps surprisingly, we will also show this intuition fails for over-parameterized problems. To proceed, given $\La$, we consider a distribution $\Dc_{\La}$, with samples $(\x^{\La},y)\sim\Dc_{\La}$ distributed as $(\La\x,y)$. Similarly, given $\Sc$, we generate a dataset $\Sc_{\La}=(\x^\La_i,y_i)_{i=1}^n$ where $\x^\La_i=\La\x_i$. We gather the data in matrix notation via \som{Consider better motivation}\som{In neural networks and kernels, design matrix may not be normalized which motivates an understanding of the impact of covariance structure, we will consider a class of problems}
\[
\X_\La=[\x^\La_1~\x^\La_2~\dots~\x^\La_n]^T\in\R^{n\times p} \quad \text{and}\quad \y=[y_1~y_2~\dots~y_n]^T \in \R^n.
\] 
To proceed, using quadratic loss, we define the empirical (training) and population (test) losses,
\begin{align}
\Lch_{\La}(\bt):=\frac{1}{n}\sum_{i=1}^n (y_i-\bt^T\x^{\La}_i)^2=\frac{1}{n}\tn{\y-\X_\La\bt}^2\quad\text{and}\quad\Lc_{\La}(\bt):=\E[(y-\bt^T\x^{\La})^2].\label{eqn: ERM linear}
\end{align}
Let $\bth^{\La}, \btb^{\La}$ be the global minima of $\Lch_{\La}$ and $\Lc_{\La}$ respectively. Let $\bSi=\E[\x\x^T]$ be the population covariance and $\bb=\E[\x y]$ be the cross-correlation. For simplicity, we assume $\bSi$ is full-rank. {\em{We will drop the subscript $\La$ when $\La=\Iden_p$}}. The solutions $\bth^{\La}, \btb^{\La}$ are given by 
\[
\bth^{\La}=\X^\dagger_{\La} \y\quad\text{and}\quad\btb^{\La}=\La^{-1}\bSi^{-1}\bb,
\] 
respectively, where $\dagger$ denotes the pseudo-inverse. The following lemma is instructive in understanding the weight importance and invariance to feature scaling for the least-squares problem above~\eqref{eqn: ERM linear}.
\vs\begin{lemma}[Pruning with Population]\label{cov importance}
	Let $\btb^{\La}$ be the minimizer of population loss and fix $\ix \subseteq [p]$. NI $\Icn_\ix(\btb^\La)$ and HI $\Ic^H_\ix(\btb^\La)$ w.r.t.~population loss $\Lc_\La$ are invariant under invertible diagonal $\La$. If the covariance $\bSi$ is also diagonal, then NI and HI are equal. In contrast, MI is $\La$ dependent via $\Ic^M_\ix(\btb^\La)=\sum_{i\in \ix}\La_{i,i}^{-2}{\btb_i}^2$ where $\btb=\btb^{\Iden_p}$ is the original model.
	
\end{lemma}
This lemma states that NI and HI are invariant to scaling and coincide when features are uncorrelated. On the other hand, MI suffers from feature scaling. As the features get larger, the corresponding weight decreases which results in an artificial decrease in importance. This highlights a fundamental shortcoming of MP and necessity of feature normalization, which was previously discussed in the literature~\cite{santurkar2018does,ioffe2015batch,jayalakshmi2011statistical,ekenel2006norm}. In Sections \ref{train loss} and \ref{neural sec}, we will see that MP fails as soon as the problem is not well-conditioned either in terms of covariance spectrum or neural network initialization.

Invariance to feature scaling is a property of over-determined problems ($n>p$) which admit unique solution (population loss is a special case with $n=\infty$). Focusing on training loss, suppose $\X\in\R^{n\times p}$ is not rank deficient. Then, the minimum-norm solution $\bth^\La$ has the form
\begin{align}
\bth^\La=\begin{cases}\La^{-1}\bth,\quad\text{when}\quad n\geq p,\\
\La\X^T(\X\La^2\X^T)^{-1}\y,\quad\text{otherwise.}\end{cases}\label{LS soln}
\end{align}
When $n\geq p$, we trivially have $(\bth^{\La})^T\x^\La=\bth^T\x$, thus $\bth$ and $\bth^\La$ achieve the exact same test/training loss. On the other hand, for over-parameterized problems ($n<p$), which is the regime of interest for  neural network pruning, this is no longer the case, and we will see that $\La$ plays a critical role in the eventual test performance as it dictates which solution the optimization problem selects.


\vs\subsection{Characterizing Pruning Performance and Covariance Bias}\label{train loss}\vs
\mc{repeat $y=\x^\top\bt+z$ on line 175 and 176}
In this section, we provide analytical formulas which enable us to compare different pruning methods and assess implicit covariance bias when $n<p$ under a realizable dataset model. Suppose $(\x_i)_{i=1}^n\distas\Nn(0,\Iden_p)$ so that $\bSi=\Iden_p$ and $\bSi_\La:=\E[\x^\La(\x^\La)^T]=\La^2$. Given a ground-truth vector $\btb\in\R^p$ (which corresponds to the population minima), we generate the labels via $y=\x^T\btb+z$ and\vs
\[
y_i=\x_i^T\btb+z_i\quad\text{for}\quad 1\leq i\leq n,
\]
where $z,(z_i)_{i=1}^n\distas\Nn(0,\sigma^2)$ are the additive noise. We will study the minimum norm least-squares solution \eqref{LS soln} which is also the solution gradient descent converges when initialized from zero. To assess pruning performance, we need to quantify the test loss of the pruned solution $\Pi_{s}(\bth)$. 

\vs
\noindent{\bf{Connection to denoising:}} We accomplish this by {\em{relating the test loss of the pruned model to the risk of a simple denoising problem}}. In essence, this denoising question is as follows: Given noisy measurements $\bt_{\text{nsy}}=\btb+\g$ of a ground-truth vector $\btb$ with $\g\sim\Nn(0,\sigma^2\Iden_p)$, what is the pruning error $\E[\tn{\Pi_s(\bt_{\text{nsy}})-\btb}^2]$? Note that this error typically doesn't have a closed form answer as hard-thresholding is not a continuous function, however, it greatly simplifies the original problem of solving least-squares. We also note that if one uses soft-thresholding (i.e.~shrinkage) operator for pruning, closed form solution is available. The fundamental connection between denoising and linear inverse problems are studied for under-parameterized least-squares and lasso regression \cite{donoho2013accurate,donoho2009message}. Our connection to denoising is established by connecting the distribution of the $\bth^\La$ to an auxiliary distribution described below.
\mc{ Line 186, "inverse problems are studied", are-> is}








\vs\vs
\begin{definition}[Auxiliary distribution] \label{aux_def}Fix $p>n\geq 1$ and set $\kappa=p/n>1$. Given $\sigma>0$, positive definite diagonal matrix $\La$ and ground-truth vector $\btb$, define the unique non-negative terms $\Xi,\Gamma,\bz\in\R^p$ and $\bg\in\R^p$ as follows\vspace{-12pt}
\begin{align}
&\Xi>0\quad\text{is the unique solution of}\quad 1=\frac{\kappa}{p}\sum_{i=1}^p\frac{1}{1+(\Xi\La_{i,i}^2)^{-1}},\label{aux solve}\\
&\Gamma=\frac{\sigma^2+\sum_{i=1}^p\bz_i^2\btb_i^2}{\kappa(1-\frac{\kappa}{p}\sum_{i=1}^p{(1+(\Xi\La_{i,i}^2)^{-1})^{-2}})},\nn\\
&\bz_i=\frac{1}{1+\Xi\La_{i,i}^2}\quad\text{and}\quad \bg_i=\frac{\kappa\sqrt{\Gamma}}{1+(\Xi\La_{i,i}^2)^{-1}}\quad\text{for}\quad 1\leq i\leq p.\nn
\end{align}
Let $\h\sim\Nn(0,\frac{1}{p}\Iden_p)$. Define the auxiliary vector $\bta^\La$ of the ground-truth $\btb$ as
\begin{align}
\bta^\La= \La^{-1}[(\onebb_p-\bz)\odot\btb+\bg\odot\h].\label{aux dist def}
\end{align}
\end{definition}
We remark that this definition can be adapted to asymptotic setup $p\rightarrow\infty$ assuming covariance spectrum converges (e.g.~discrete sum over entries converges to an integral). In the special case of identity covariance ($\bSi=\Iden_p$), $\bta$ reduces to $\bta=\frac{1}{\kappa}\btb+\sqrt{\frac{\sigma^2}{\kappa-1}+\frac{(\kappa-1)\tn{\btb}^2}{\kappa^2}}\h$. This distribution arises from applying Convex Gaussian Min-Max Theorem (CGMT) \cite{Gor,Gor2,thrampoulidis2015regularized,OymLAS,thrampoulidis2018precise} to over-parameterized least-squares. CGMT provides a framework for predicting the asymptotic properties of optimization problems involving random matrices by connecting them to simpler auxiliary optimizations involving random vectors (some example applications \cite{montanari2019generalization,deng2020model,salehi2018precise}). Thus, based on CGMT, $\bth^\La$ and the auxiliary vector $\bta^\La$ are expected to have similar distributional properties and $\bta^\La$ can be used as a proxy to capture the properties of $\bth^\La$. In supplementary, we discuss to what extent this distributional similarity can be formalized (e.g.~for Lipschitz functions). Note that, after solving for $\bz,\bg$ in \eqref{aux solve}, we can sample from the auxiliary distribution which is a noisy version of $\btb$ which connects us to denoising. To proceed, our analytic formulas for the test error of an $s$-sparse model via MP and HP takes the following form:
\begin{align}
\text{MP loss:}~\E_{\h}[\tn{\La\Pi^M_s(\bta^\La)-\btb}^2]+\sigma^2,\quad \text{HP loss:}~\E_{\h}[\tn{\Pi^M_s(\La\bta^\La)-\btb}^2]+\sigma^2.\nn
\end{align}


Next, we verify our performance prediction and study the role of covariance structure $\La$. We generate $\btb$ with polynomially decaying entries, specifically $\btb_i=1/(1+4i/p)^2$, and then scale it to unit Euclidian norm. Recall that original covariance is identity, thus initial larger entries of $\btb$ are more important for population risk. In our experiments, we parameterize $\La$ by a scalar $\la$ and set it as 
\begin{align}
\La_{i,i}=\begin{cases}\la\quad\text{if}\quad 1\leq i\leq p/10,\\1\quad\text{if}\quad i>p/10.\end{cases}\label{linear scaling}
\end{align}
This choice modifies the most important 10\% weights of the problem. We consider $\la\in \{1/2,1,5\}$. As formalized in Thm.~\ref{lem win}, when $\la>1$, we expect a positive covariance bias since important weights are aligned with the principal directions of the covariance and are easier to learn. In Figures \ref{fig:positive} and \ref{fig:negative}, the lines are the analytical predictions based on Definition \ref{aux_def} and the markers are performance of the actual least-squares solution which nicely match for all pruning methods and $\la$. Figure \ref{fig:positive} contrasts $\la=1$ and $\la=5$. For $\la=1$, MP and HP coincide as the Hessian is identity. However when $\la=5$, HP performs much better than $\la=1$ for all sparsity levels. MP drastically fails for small sparsity levels as the initial weights of $\bth^\La$ are important but small due to the $\la$-scaling thus MP inaccurately ignores them. Decreasing magnitudes with increasing $\la$ is more clear for under-parameterized case (via \eqref{LS soln}) however $\La^{-1}$ dependence is also visible in \eqref{aux dist def}. Fig \ref{fig:negative} additionally highlights $\la=1/2$ which reduces the covariance and scales up the coefficients of the important weights. This leads to a negative bias because covariance structure guides the solution away from important weights. While both MP and HP performs worse than $\la=1$ case, HP performs worse due to additional penalization of the initial important weights. Finally, covariance bias is visualized in Figure \ref{fig:covariance} which displays the test NI (for $\btb$) and the training NI's (for $\bth^\La$) of the first $s$ weights. When $\la=5$, initial weights, which are important for test, have a larger training NI. As $\la$ gets smaller, remaining weights, which are not as important for test, have larger say during training and pruning performance degrades. Our Theorem \ref{lem win} formalizes these by quantifying the contributions of different weights during training.
\mc{I made the label of fig 1 a bit smaller}
	
\begin{figure}[t!]
	\centering
	\begin{subfigure}{2.1in}
	\begin{tikzpicture}
	\node at (0,0) {\includegraphics[scale=0.25]{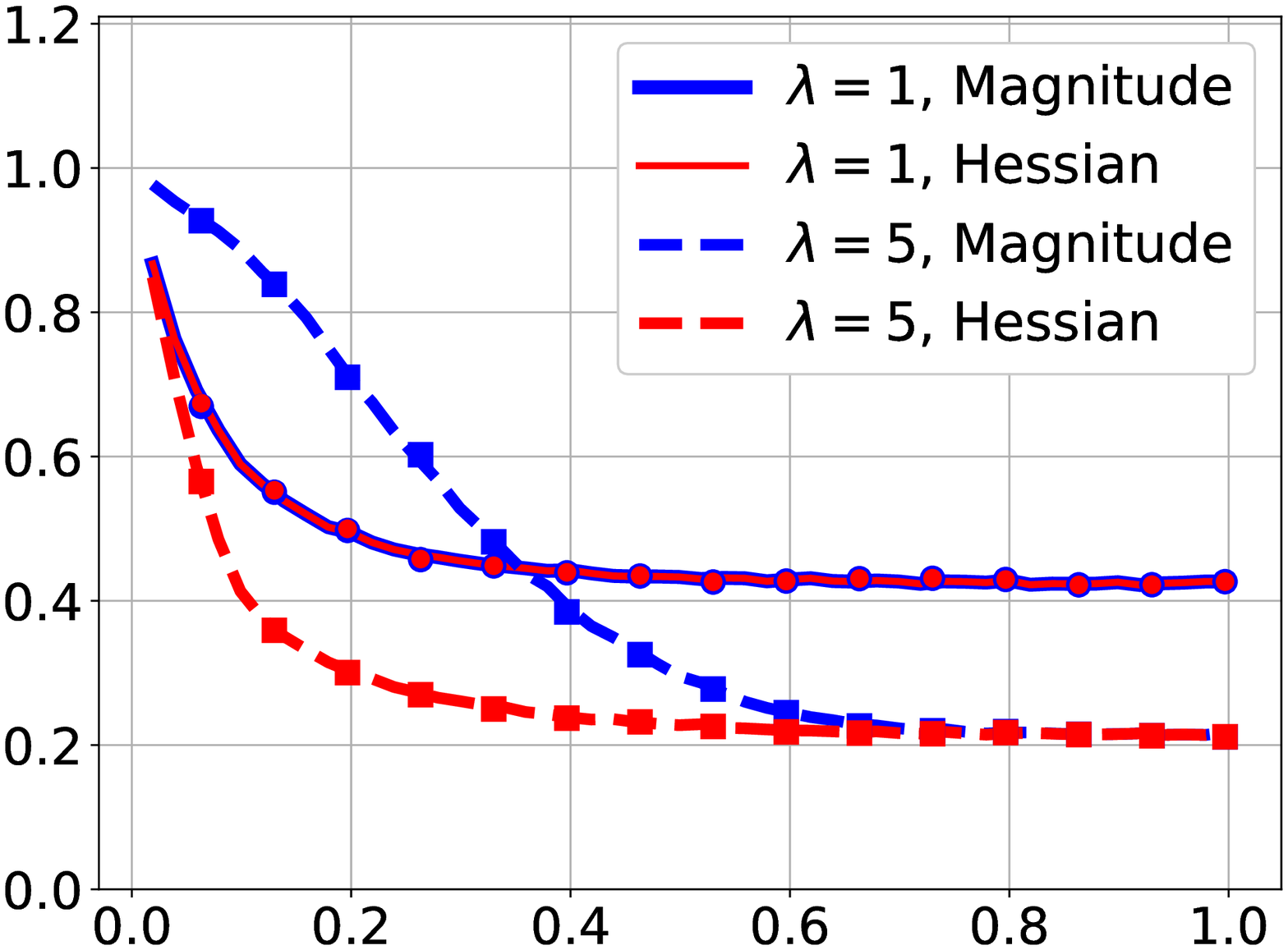}};
	\node at (-2.7,0) [rotate=90,scale=1.]{Test loss};
	\node at (0,-2.1) [scale=1.]{Fraction of non-zero ($s/p$)};
	\end{tikzpicture}\vspace{-3pt}\caption{
	\small{Pruning with $\la\in\{1,5\}$}}\label{fig:positive}\vspace{-0.2cm}
	\end{subfigure}
	\begin{subfigure}{2.1in}
	\begin{tikzpicture}
	\node at (0,0) {\includegraphics[scale=0.25]{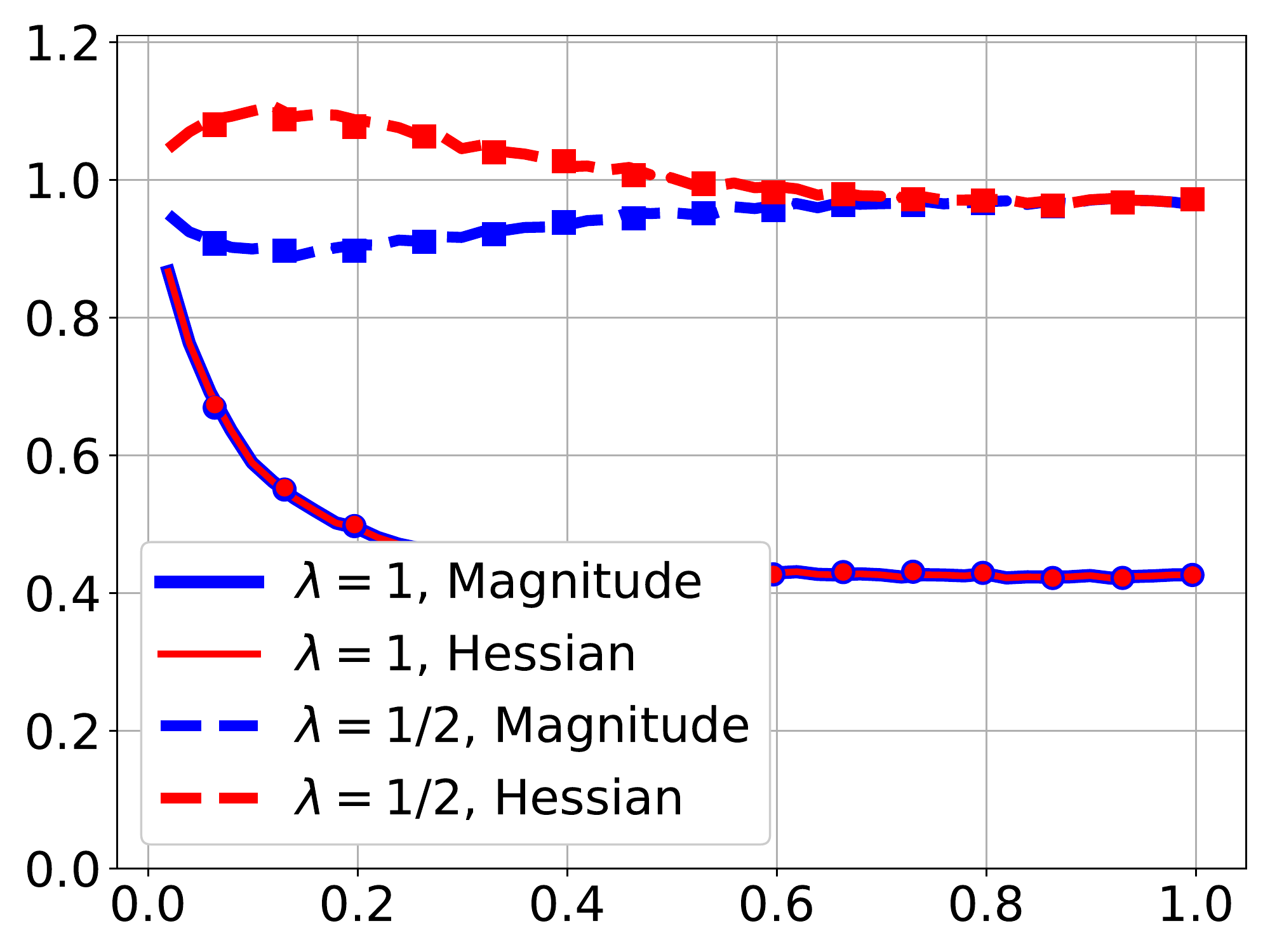}};
	\node at (-2.7,0) [rotate=90,scale=1.]{Test loss};
	\node at (0,-2.1) [scale=1.]{Fraction of non-zero ($s/p$)};
	\end{tikzpicture}\vspace{-3pt}\caption{
	\small{Pruning with $\la\in\{1,1/2\}$}}\label{fig:negative}\vspace{-0.1cm}
		\end{subfigure}
	\begin{subfigure}{2.1in}
	\begin{tikzpicture}
	\node at (0,0) {\includegraphics[scale=0.25]{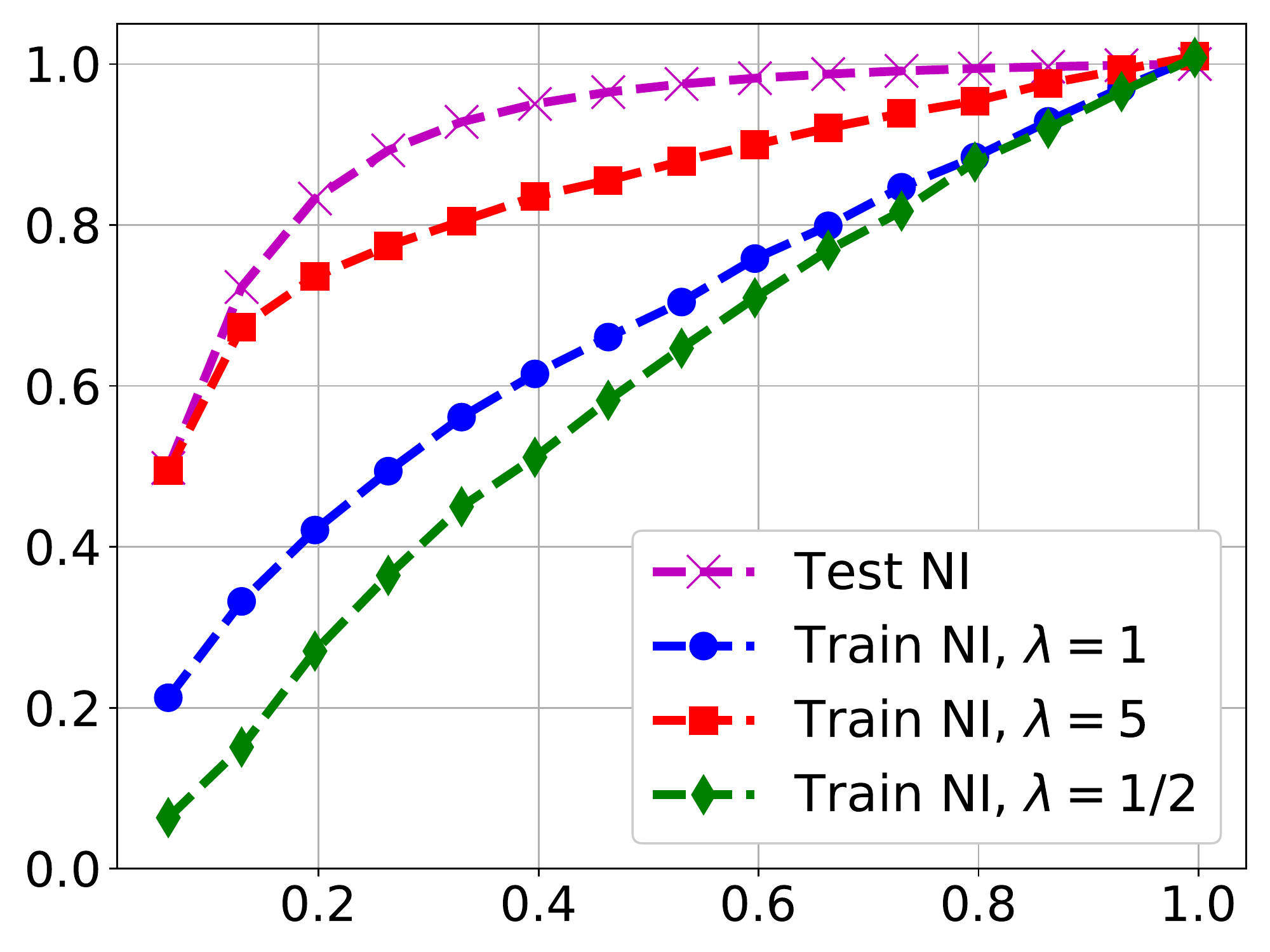}};
	\node at (-2.7,0) [rotate=90,scale=1.]{Natural Importance};
	\node at (0,-2.1) [scale=1.]{First $s/p$ entries};
	\end{tikzpicture}\vspace{-3pt}\caption{\small{NI of the first $s$ entries}}
	\label{fig:covariance}\vspace{-0.2cm}
	\end{subfigure}
		\caption{\small{In (a) and (b), the lines are the analytical prediction from Def~\ref{aux_def} and markers are the actual min-norm solution. $p=1000, \kappa=p/n = 5/3$ and $\sigma = 0.1$. (c) Natural importance associated with the first $s$ weights. For $\la=5$, (a) HP achieves better performance and (c) training and test NI have a better match.}}\vspace{-15pt}
\end{figure}

\vs\vs\section{Hessian Bias and the Role of Initialization for Neural Nets}\label{neural sec}\vs\vs
This section extends our discussion of importance and pruning to another fundamental model class: neural networks with one-hidden layer. Suppose input dimension is $d$, output dimension is $K$ and the network has $m$ hidden units. Such a network with ReLU activation is given by $f_{\bt}(\x)=\V\act{\W\x}$, where $\W\in\R^{m\times d}$ and $\V\in\R^{K\times m}$ are the input and output layers respectively and $\bt=(\W,\V)\in\R^{p=(d+K)m}$ is the vector composed of the entries of $\W,\V$. Let $\ixi$ and $\ixo$ denote the index of the entries of $\W,\V$ in $\bt$. Given a dataset $\Sc=(\x_i,y_i)_{i=1}^n$ and loss $\ell$, we minimize\vs\vs
\begin{align}
\Lc(\bt)=\frac{1}{n}\sum_{i=1}^n\ell(y_i,f_{\bt}(\x_i)).\label{loss}
\end{align}
{\bf{Equivalent networks:}} To study neural net pruning and initialization, we shall consider a class of networks $\bt^\la=(\la\W,\la^{-1}\V)$ generated from a base network $\bt^1=(\W,\V)$. Observe that all vectors $\bt^\la$ implement the same function due to the linearity of ReLU however magnitudes of layers are varying. The following lemma shows how the parameter $\la$ affects MI, HI, and Hessian.
\mc{MI,HI and partial derivatives obey --Line 236}

\begin{lemma} \label{simple hessian gradient}Consider the loss \eqref{loss} and class of networks $(\bt^\la)_{\la>0}$. For all $\la>0$, MI, HI and partial Hessians w.r.t.~input/output layer weights $\W,\V$ obey
\begin{align}
&\Ic^M_\ixi(\bt^\la)=\la^2\Ic^M_\ixi(\bt^1)\quad\text{and}\quad\Ic^M_\ixo(\bt^\la)=\la^{-2}\Ic^M_\ixo(\bt^1),\nn\\
&\Ic^H_\ixi(\bt^\la)=\Ic^H_\ixi(\bt^1)\quad\text{and}\quad\Ic^H_\ixo(\bt^\la)=\Ic^H_\ixo(\bt^1),\label{constant HI}\\
&\frac{\pa^2}{\pa^2{\W}}\Lc(\bt^\la)=\la^{-2}\frac{\pa^2}{\pa^2{\W}}\Lc(\bt^1)\quad\text{and}\quad\frac{\pa^2}{\pa^2{\V}}\Lc(\bt^\la)=\la^2\frac{\pa^2}{\pa^2{\V}}\Lc(\bt^1).\label{gradient change}
\end{align}
\end{lemma}
In words, increasing $\la$ increases MI, preserves HI, and decreases the Hessian magnitude for the input layer and has the reverse effect on the output layer. Suppose we train the network from initializations $\bt^\la$ on \eqref{loss}. What happens at the end of the training as a function of $\la$? Does eventual MI and HI exhibit similar behavior to initialization? What about NI?

\mc{Don't know whether "interpolation to training data" is correct grammar or not --Line 244/245.}

To answer these, in Figure \ref{fig:Importance_layer}, we conduct an empirical study on MNIST by training a one-hidden layer network with cross-entropy loss. Here $m=1024$, $d=784$ and $K=10$. We set $\bt^{1}_{\ini}=(\W_\ini,\V_\ini)$ with \textit{He normal} initialization \cite{he2015delving}. We then train networks with $\la$-scaled initializations $\bt^\la_\ini=(\la\W_\ini,\la^{-1}\V_\ini)$. Let $\bt_\fn^{(\la)}=(\W_\fn^{(\la)},\V_\fn^{(\la)})$ be the final model obtained by training until interpolation to training data (or maximum 150 epochs). Figures~\ref{fig:Importance_layer_1} and \ref{fig:Importance_layer_2} display MI, HI, and NI for input and output layers respectively. Here, for NI, we use {\em{Init-Pruning}} and quantify importance of a layer (e.g.~$\W_\fn^{(\la)}$) with respect to its initial weights (e.g.~$\W_\ini^{\la}=\la\W_\ini$). Observe that, regular pruning is not informative as setting a layer to zero kills the network output.



\noindent{\bf{Understanding MI and HI:}}\mc{I changed legend's position and text for fig 2c. change w,v to input NI, output NI} Figures~\ref{fig:Importance_layer_1} and \ref{fig:Importance_layer_2} show that initial and final MI exhibit a near perfect match. The initial HI stays constant as predicted by Lemma~\ref{simple hessian gradient}. Final HI increases with $\la$ for both layers, however it can be verified that the ratio of HI between input and output layers is approximately preserved. Perhaps surprisingly, Lemma~\ref{simple hessian gradient} seems to predict not only the initial importance but also the MI/HI of the final network. Fortunately, this can be mostly explained by the optimization dynamics of wide and large networks where gradient descent finds a global optima close to initialization and {\em{final weights (and Hessian) do not deviate much from initial ones}} \cite{chizat2019lazy,arora2019fine,oymak2020towards,du2018gradient,jacot2018neural,allen2019convergence,li2019gradient}.\mc{\cite{lee2019wide} also about wide network and linearization -Line 256}

\mc{No more $10^5$HI, but "we draw the blue line as $\sqrt{MI}$ which is proportional to $\la$" --Line 249/250}

\vs
In Figures~\ref{fig:compare_purining_C10_1} and \ref{fig:compare_purining_C10_2}, we first prune $\bt_\fn^{(\la)}$ to a fixed nonzero fraction and then retrain the pruned weights from the same initialization (i.e.~\cite{frankle2019lottery}). MP is only competitive with HP when $\la=1$ where input and output layer entries have similar magnitude due to He initialization. In Fig.~\ref{fig:compare_purining_C10_1}, as $\la$ grows output layer becomes small and gets fully pruned. As $\la$ gets smaller, eventually input layer is fully pruned. Here, what is rather remarkable is the robustness of HP for full range of $\la$ choices which arises from \eqref{constant HI}. Arguably, HI being invariant to $\la$ makes it more attractive than NI as it avoids the issue of {\em{layer death}} i.e.~all of the weights in a layer getting pruned. Figure \ref{fig:compare_purining_C10_3} visualizes the fraction of unpruned weights in input and output layers for various $\la$. HP (solid) curves are stable whereas MP (dotted) curves are highly volatile and easily hit zero except a narrow region. We note that, an alternative way of avoiding layer death is pruning layers individually. Supplementary provides further experiments on this for completeness.
\mc{didn't find layer death reference -Line 263}


 
 \vs
\noindent{\bf{Understanding NI and optimization dynamics:}} If our shallow network is sufficiently wide, each layer (or large groups of weights) can individually fit the training dataset. This can be viewed as a competition between the layers and a natural question is how much a layer contributes to the learning. This question is answered by NI. In Figure \ref{fig:Importance_layer_1} orange line displays the change in input layer NI (with $\Lc$ of Def.~\ref{def: weight impo} is training loss) which demonstrates that NI is decreasing function of $\la$ and moves in the opposite direction to MI. Figure \ref{fig:Importance_layer_test} verifies the same NI behavior for test loss and test error. Specifically, for large $\la$, input layer is responsible for most of the test accuracy and for small $\la$, it is the output layer. \mc{"it is the output layer" seems a bit strange, can we change to "output layer takes the role" --Line 273/274} Our key technical contribution in this section is {\em{providing a theoretical explanation to this NI behavior and relating it to optimization dynamics}}. In essence, we will connect NI to the only feature in Lemma \ref{simple hessian gradient} that exhibit similar behavior, the Hessian. Below we state our result on the Hessian and NI relation in terms of Polyak-Lojasiewicz (PL) condition \cite{karimi2016linear}.





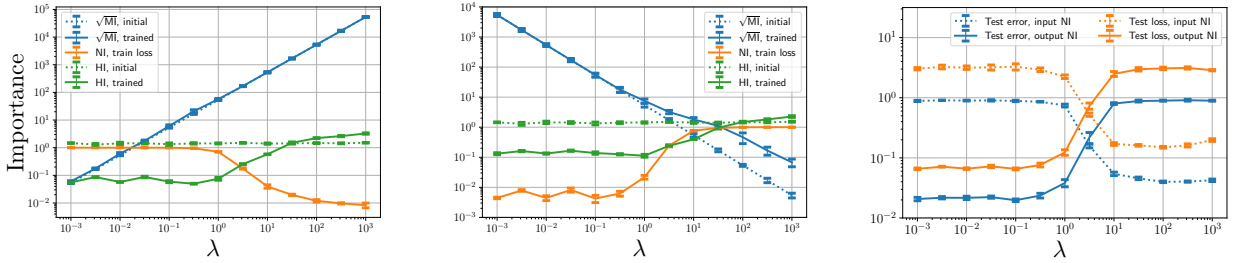
\begin{figure}[t!]
	\begin{subfigure}{2.2in}\vspace{-5pt}
		\centering
		\begin{tikzpicture}
		\node at (0,0) {\includegraphics[scale=0.365]{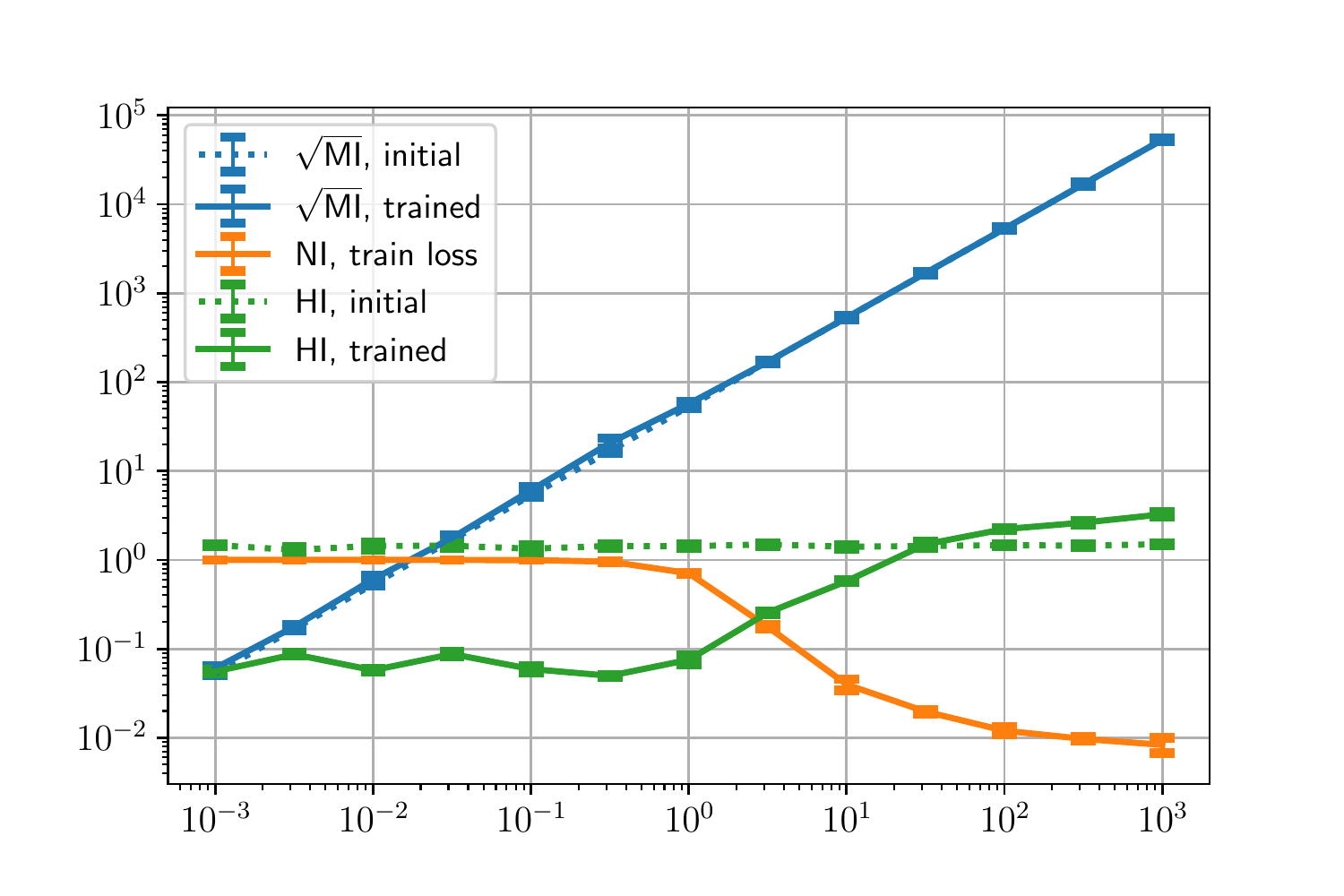}};
		\node at (-2.6,0) [rotate=90,scale=.9]{Importance};
		\node at (0,-1.8) [scale=.9]{$\lambda$};
		\end{tikzpicture}\caption{\small{Importance of input layer $\W_\fn^\la$}}\label{fig:Importance_layer_1}\vspace{-0.2cm}
	\end{subfigure}~\begin{subfigure}{2.2in}\vspace{-5pt}
	\centering
	\begin{tikzpicture}
	\node at (0,0) {\includegraphics[scale=0.365]{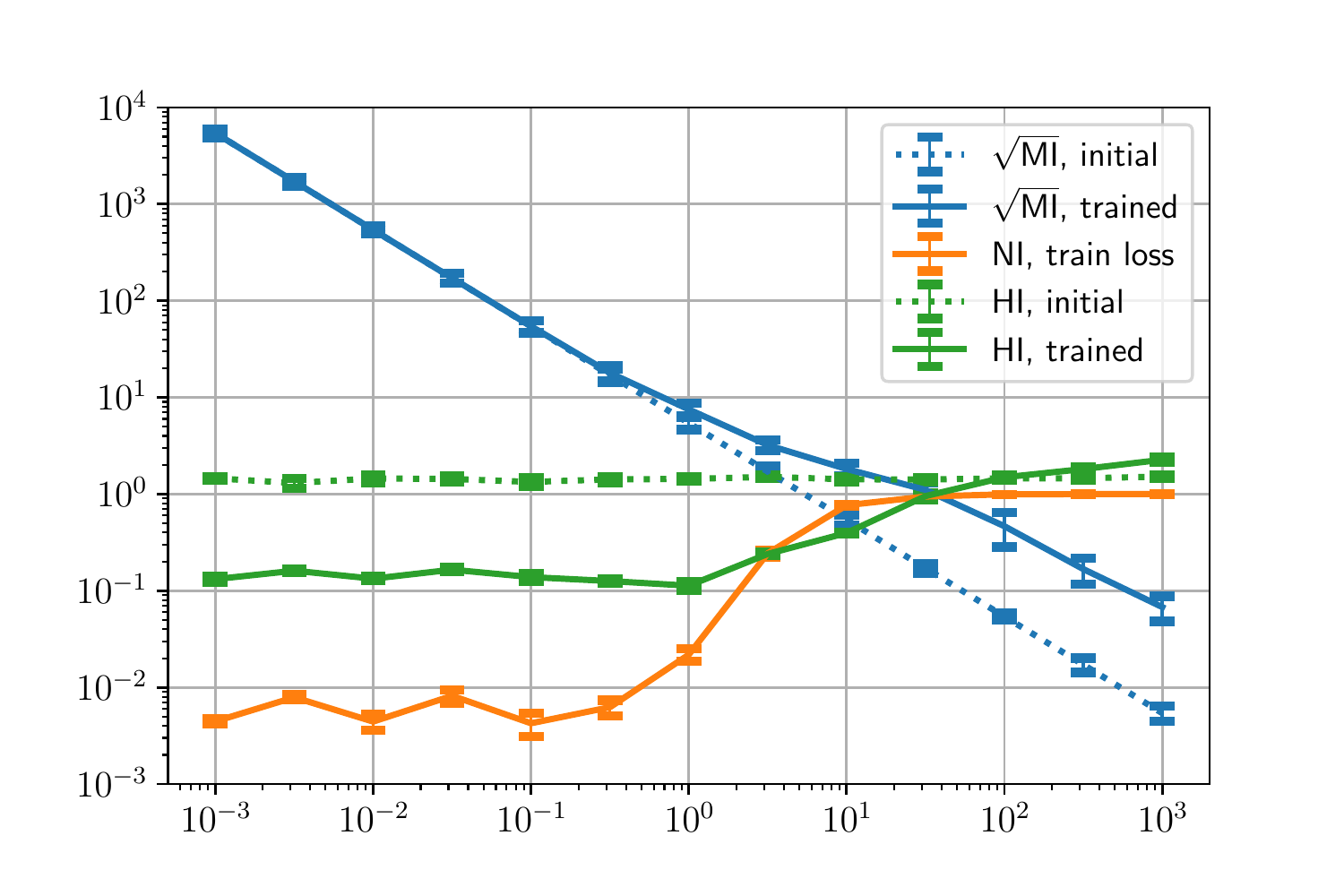}};
	\node at (0,-1.8) [scale=.9]{$\lambda$};
	\end{tikzpicture}\caption{\small{Importance of output layer $\V_\fn^\la$}}\label{fig:Importance_layer_2}\vspace{-0.2cm}
\end{subfigure}\begin{subfigure}{2.2in}\vspace{-5pt}
\centering
	\begin{tikzpicture}
	\node at (0,0) {\includegraphics[scale=0.365]{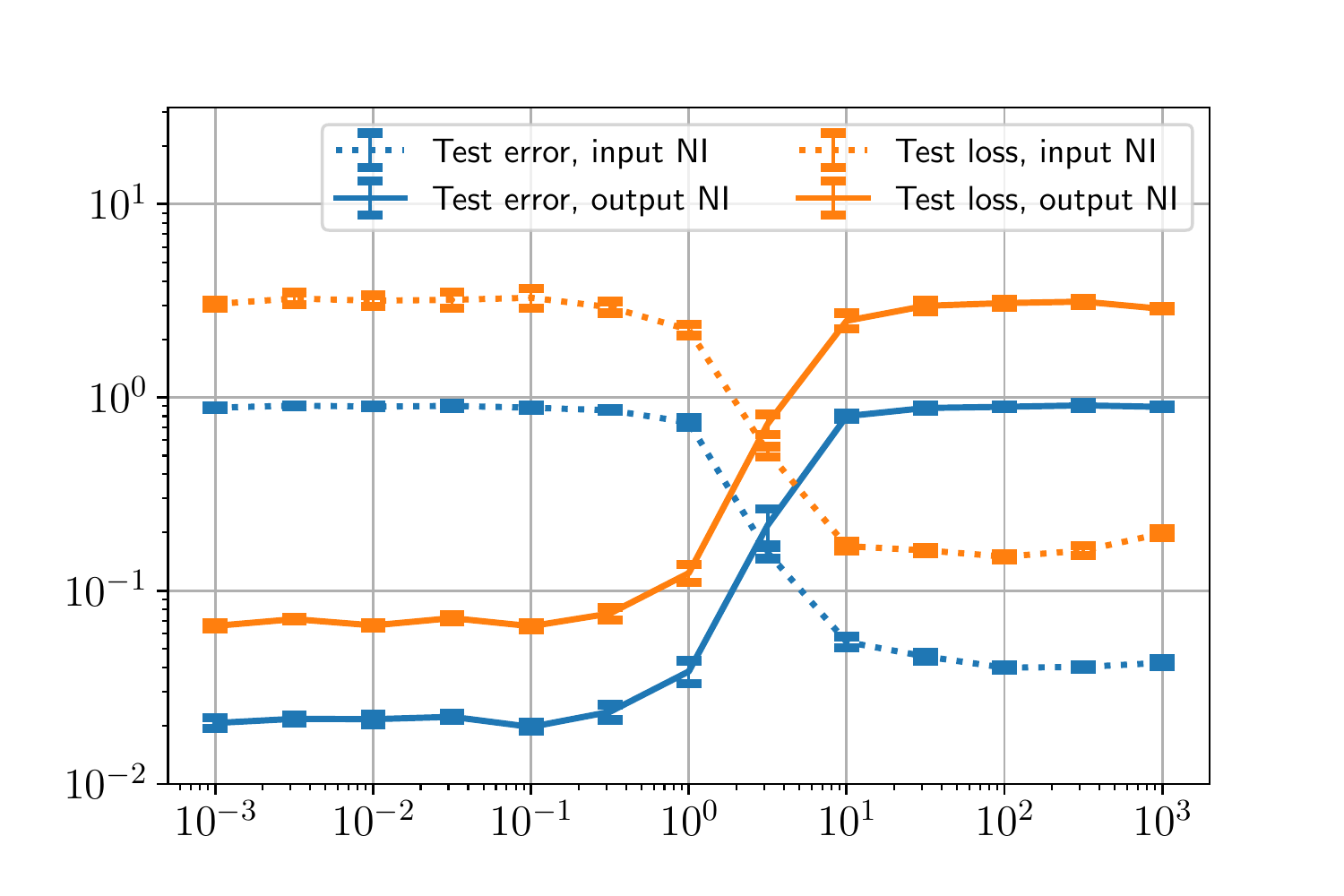}};
	\node at (0,-1.8) [scale=.9]{$\lambda$};
	\end{tikzpicture}\caption{\small{Natural importances w.r.t.~test}}\label{fig:Importance_layer_test}\vspace{-0.2cm}
\end{subfigure}\caption{\small{(a) and (b) show the comparison of importance criterias for input and output layers when training with a shallow network with cross entropy and with initializations $\bt_\ini^\la=(\la\W_\ini,\la^{-1}\V_\ini)$. (c) shows the NI w.r.t.~test classification error and loss obtained by setting one of the layers to its initialization. }}\label{fig:Importance_layer}\vspace{-15pt}
\end{figure}

\vs\begin{definition}[Partial PL and Smoothness (PPLS)] Let $\Lc(\bt)$ be a loss function satisfying $\min_{\bt}\Lc(\bt)=0$. Given an index set $\Delta\subset[p]$, we say that PPLS holds with parameter $L\geq \mu\geq 0$ if partial derivative $\frac{\pa }{\pa \bt_{\Delta}}\Lc(\bt)$ is $L$-Lipschitz function of $\bt_{\Delta}$ and obeys $\tn{\frac{\pa }{\pa \bt_{\Delta}}\Lc(\bt)}^2\geq 2\mu\Lc(\bt)$.
\end{definition}\vs
While PL allows for non-convex optimization, when specialized to strong convexity, Partial PL condition provides a lower bound on the submatrix of Hessian induced by the set $\ix$. Regular PL condition guarantees global convergence of gradient descent, thus if PPLS holds over $\ix$, training only over $\ix$ is sufficient to achieve zero loss. A good example of PPLS is linear regression with two feature sets $\X_1\in\R^{n\times p_1}$ and $\X_2\in\R^{n\times p_2}$ with $p_1,p_2\geq n$ where we fit \mc{ to be consist, dim of X1 and X1 should be d1 d2? I feel both are ok but which is better?}
\begin{align}
\Lc(\bt)=\min_{\bt=[\bt_1~\bt_2]}0.5\tn{\y-\X_1\bt_1-\X_2\bt_2}^2.\label{lin reg}
\end{align}
$\Lc$ satisfies PPLS over $[p_1]=\{1,\dots,p_1\}$ with parameters $L_1=\|\X_1\|^2$ and $\mu_1=\sigma_{\min}(\X_1)^2$. For randomly initialized over-parameterized networks, each layer solves a kernel regression and would satisfy PPLS under mild conditions on the dataset \cite{chizat2019lazy,du2018gradient,jacot2018neural,allen2019convergence,oymak2020towards} \mc{also \cite{lee2019wide}} . Specifically, linearized neural network dynamics on $\bt=(\W,\V)$ connects \mc{connects-> connect} to the regression task \eqref{lin reg} via the Taylor expansion around initialization where input and output layers have linearized features arising from the Jacobian map given by $\X_{\W}=[\frac{\pa f(\x_1)}{\pa \W}~\dots~\frac{\pa f(\x_n)}{\pa \W}]^T\in\R^{n\times md}$ and $\X_{\V}=[\frac{\pa f(\x_1)}{\pa \V}~\dots~\frac{\pa f(\x_n)}{\pa \V}]^T\in\R^{n\times Km}$. The following theorem provides a theoretical explanation of NI behavior via PPLS by quantifying relative contributions of different sets of weights.
\vs
\mc{ added arrow for fig 3a, adjusted size of 3c} 
\begin{theorem}[Larger Hessian Wins More]\label{lem win}
Suppose the entries of $\bt\in\R^p$ are union of $D$ non-intersecting sets $(\Delta_i)_{i=1}^D\subset [p]$ and PPLS holds over $\ix_i$ with parameters $L_i\geq \mu_i\geq 0$ for all $i$. Set $\mu=\sum_{i=1}^D\mu_i$ and $L=\sum_{i=1}^DL_i$. Starting from a point $\bt_0$, and using a learning rate $\eta\leq 1/L$, run gradient iterations $\bt_{\tau+1}=\bt_\tau-\eta \grad{\bt_\tau}$. For all iterates $\tau$, the loss obeys $\Lc(\bt_\tau)\leq (1-\eta\mu)^\tau \Lc(\bt_0)$. Furthermore, setting $\kappa=L_i/\mu$, the following bounds hold for $\iii$ and $\iib=[p]-\iii$ for all $\tau\geq 0$
\begin{align}
&\tn{\bt_{\iii,\tau}-\bt_{\iii,0}}^2\leq 8{\kappa}\Lc(\bt_0)/\mu,\label{dist bound}\\
&{\Icn_{\iii}(\bt_\tau,\bt_0)}/{\Lc(\bt_0)}\leq 8\kappa^2+4{\kappa}(1-\eta\mu)^{\tau/2},\label{bad2}\\
&{\Icn_{\iib}(\bt_\tau,\bt_0)}/{\Lc(\bt_0)}\geq 1-8\kappa^2-4{\kappa}-(1-\eta\mu)^\tau\label{good2}.
\end{align}
\end{theorem}\vs
In words, this theorem captures the NI of a subset of weight\mc{weight->weights} throughout the training via the upper and lower bounds \eqref{bad2} and \eqref{good2}. For the experiments in Fig.~\ref{fig:Importance_layer}, based on \eqref{gradient change} of Lemma \ref{simple hessian gradient}, PPLS parameters $(\mu_{\W},L_{\W})$ of the input layer decay as $\la^{-2}$ and output layer parameters grow as $\la^2$. Thus, assuming $\la\geq 1$, for output layer we have $\kappa=L_{\W}/(\mu_{\W}+\mu_{\V})\sim \la^{-4}$ and, using \eqref{bad2} with $\tau=\infty$, NI is expected to decay as $\kappa^2\sim\la^{-8}$ (e.g.~for quadratic loss). Similarly, NI of the input layer is lower bounded via \eqref{good2} which grows as $1-{\cal{O}}(\la^{-4})$. Finally, for small $\la$, we have the reversed upper/lower bounds. In summary, our Theorem \ref{lem win} successfully explains the empirical NI behavior in Fig.~\ref{fig:Importance_layer}.

\vs
\eqref{dist bound} generalizes the ``short distance from initialization'' results of \cite{oymak2019overparameterized,gupta2019path} by controlling individual subsets of weights and also provides a bound on MI when $\bt_0=0$. As explained in supplementary, {\em{this theorem is tight up to local ($L_i/\mu_i$) and global ($L/\mu$) condition numbers and accurately captures the relative contributions of the weights $(\bt_{\iii})_{i=1}^D$}}. Observe that this theorem considers the Init-Pruning (w.r.t.~$\bt_0$) which is better suited for assessing optimization dynamics.

Note that the bounds of Thm \ref{lem win} greatly simplify at the global minima ($\tau\rightarrow\infty$). As mentioned earlier, training NI of Figure \ref{fig:covariance} can be explained by Thm~\ref{lem win}. In essence, scaling up a set of features increase their covariance (and PPLS parameter $\mu$) increasing the NI w.r.t.~training loss.

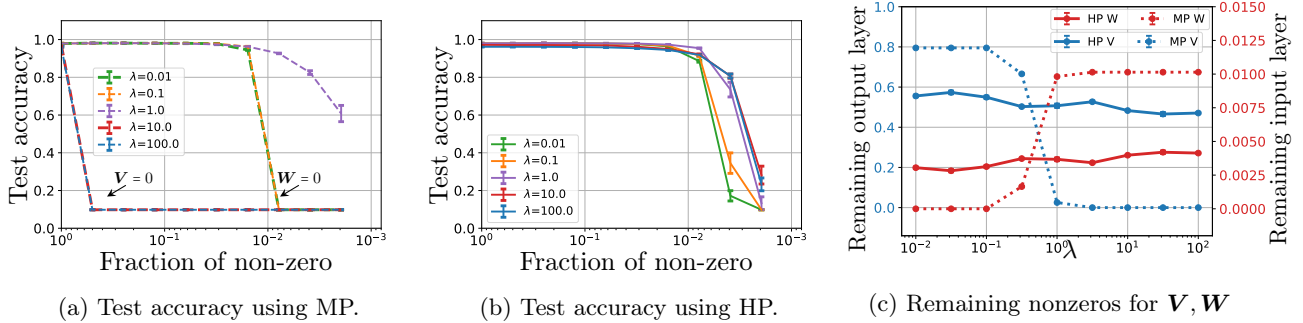
\begin{figure}[t!]
	\begin{subfigure}{2.2in}
		\centering
		\begin{tikzpicture}[scale=1,every node/.style={scale=1}]
		\node at (0,0) {\includegraphics[scale=0.36]{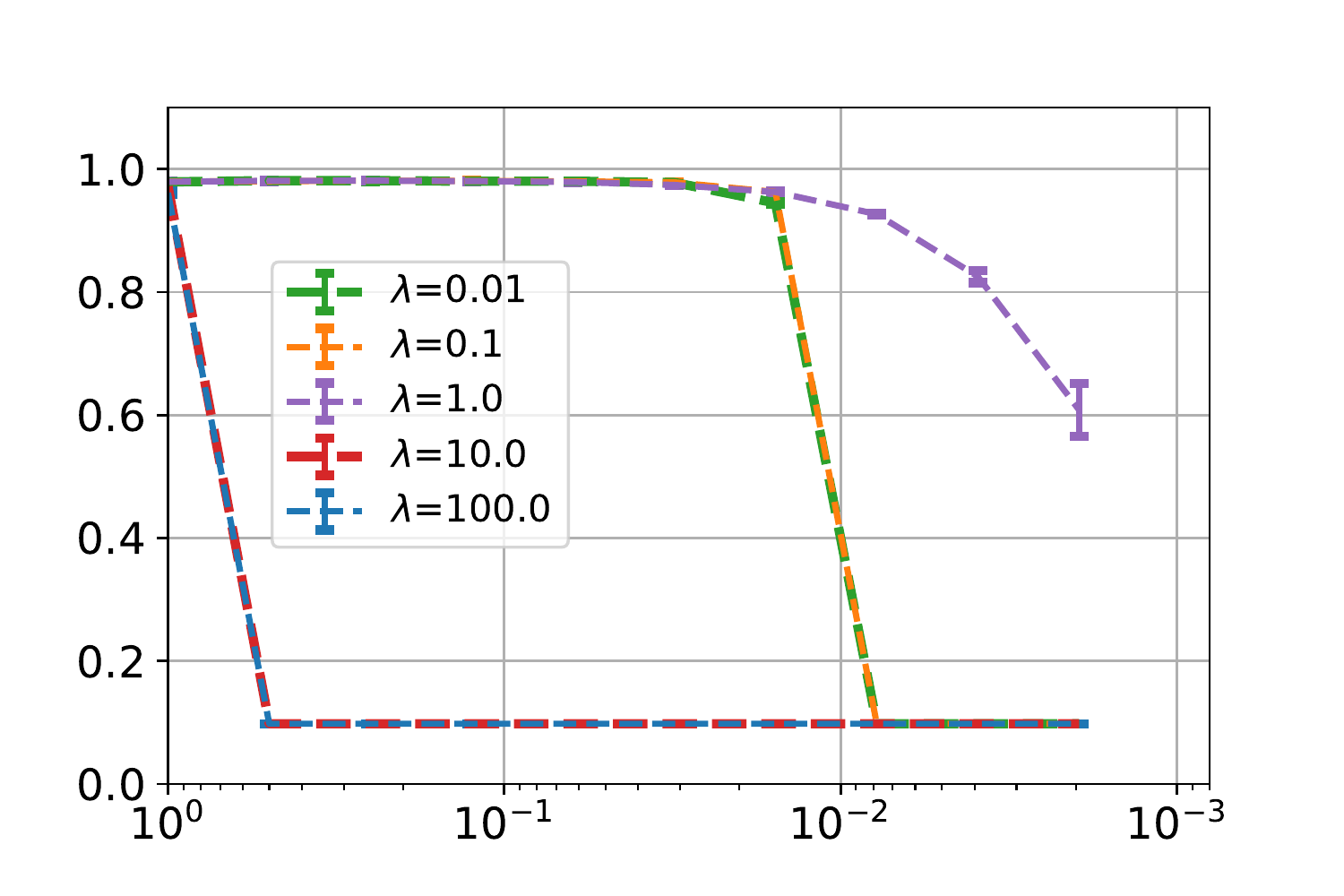}};
		\node at (-2.6,0) [rotate=90,scale=1.]{Test accuracy};
		\node at (0,-1.8) [scale=1.]{Fraction of non-zero};
		\node at (-1.1,-0.7) [scale=.6] {$\V=0$};
		\node at (1.1,-0.7) [scale=.6] {$\W=0$};
		\draw [->,>=stealth] (-1.2,-0.8) -- (-1.45,-0.95);
		\draw [->,>=stealth] (1.1,-0.8) -- (0.85,-0.95);
		\end{tikzpicture}\caption{\small{Test accuracy using MP. }}\label{fig:compare_purining_C10_1}
	\end{subfigure}\begin{subfigure}{2.2in}
		\centering
		\begin{tikzpicture}[scale=1,every node/.style={scale=1}]
		\node at (0,0) {\includegraphics[scale=0.36]{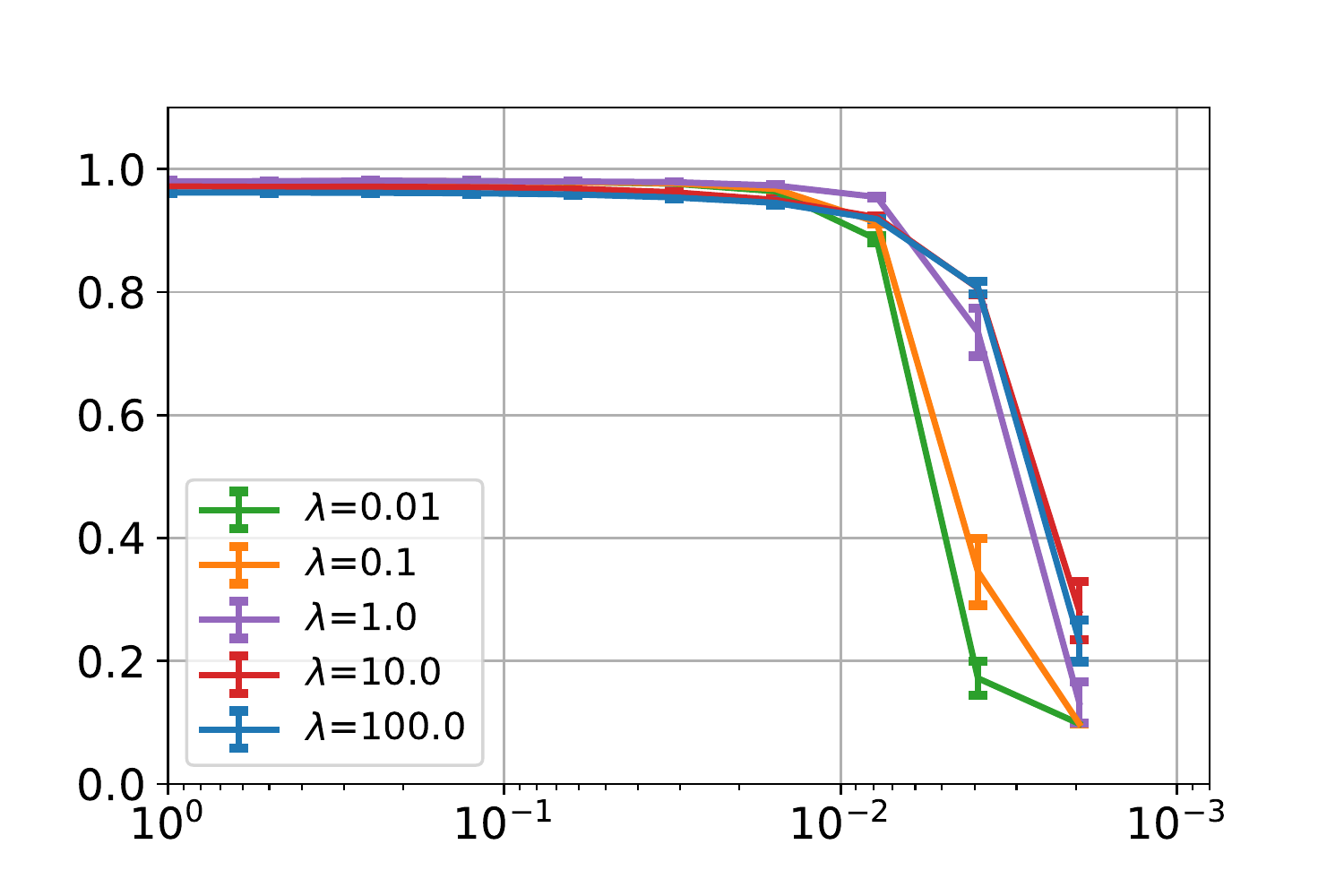}};
		\node at (-2.6,0) [rotate=90,scale=1.]{Test accuracy};
		\node at (0,-1.8) [scale=1.]{Fraction of non-zero};
		
		\end{tikzpicture}\caption{\small{Test accuracy using HP.}}\label{fig:compare_purining_C10_2} 
	\end{subfigure}\begin{subfigure}{2.2in}
		\centering
		\begin{tikzpicture}[scale=1,every node/.style={scale=1}]
		\node at (0,0) {\includegraphics[scale=0.36]{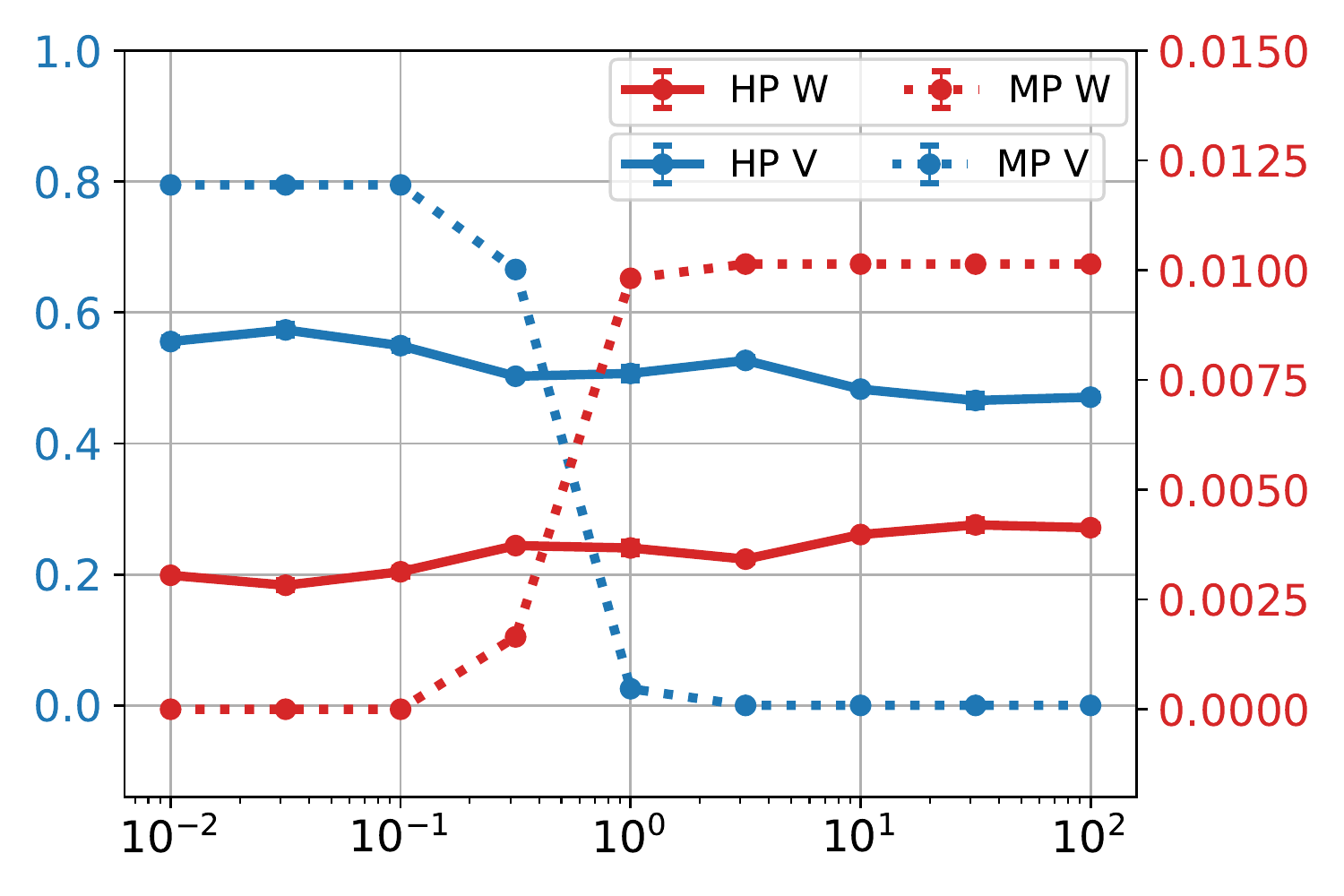}};
		\node at (-2.8,0) [rotate=90,scale=0.9]{Remaining output layer };
		\node at (2.8,0) [rotate=90,scale=0.9]{Remaining input layer };
		\node at (0,-1.6) [scale=1.]{$\lambda$};
		\end{tikzpicture}\caption{\small{Remaining nonzeros for $\V,\W$}}\label{fig:compare_purining_C10_3} 
	\end{subfigure}\caption{\small{In (a) and (b), we first apply MP and HP on the network weights $\bt_\fn^{(\la)}$ for varying pruning levels $s/p$ where $p=(K+d)m$. We then retrain the pruned network from same initial nonzeros (lottery initialization of \cite{frankle2019lottery}) and display the test accuracy. HP is more stable compared to MP under $\la$-scaled initializations. (c) Visualization of the remaining fractions of nonzero weights in input (red) and output (blue) layers after pruning the network to 1\% sparsity. Nonzero counts in both layers are stable under HP but rapidly change in MP as a function of $\la$.}}\label{fig:compare_pruning_C10}\vspace{-12pt}
\end{figure}

\vs\vs\vs\section{Conclusion}\label{sec:conc}\vs\vs\vs
We provided a principled exploration of model pruning for linear models and shallow networks. Our work reveals and formalizes the importance of Hessian/covariance structure for pruning over-parameterized models. We found that magnitude-based pruning is very brittle and requires good normalization whereas Hessian-based pruning is robust to problem structure. We also derived the first analytical performance formulas exactly capturing pruning for linear models which enabled us to do a thorough comparison between different methods. There are several interesting open directions. Can we derive similar sharp performance bounds for pruning random features or neural networks? What are the optimal initialization strategies for deep nets to enable ideal pruning performance?


{
	\bibliography{Bibfiles}
	\bibliographystyle{acm}
}
\clearpage
\appendix
\section*{Organization of the Supplementary Material}
Supplementary material is organized as follows.
\begin{enumerate}
\item Appendix \ref{app cgmt} derives Auxiliary Distribution (Definition \ref{aux_def}). We also provide the relevant background and supporting results on Convex Gaussian Min-Max Theorem (CGMT) and discuss how distributional similarity based on Def.~\ref{aux_def} can be formalized.
\item Appendix \ref{sec winner} proves Theorem \ref{lem win}. In Appendix \ref{subsec tight} (see Proposition \ref{tight}), we also provide theoretical results proving the tightness of the bounds provided in Theorem \ref{lem win}.
\item Appendix \ref{app lemma proof} proves Lemmas \ref{cov importance} and Lemma \ref{simple hessian gradient}.
\item Appendix \ref{app layerwise} provides further numerical results on Section \ref{neural sec}. Appendix \ref{app layerwise} provides results on layer-wise pruning, where pruning is done on each layer individually, and compares to Section \ref{neural sec} which uses standard pruning.
\item Appendix \ref{app further} provides further technical results supporting Appendix \ref{app cgmt}.
\end{enumerate}
\section{Auxiliary Distribution for Pruning Linear Models}\label{app cgmt}
\subsection{Technical Background on Convex Gaussian Min-Max Theorem}
CGMT framework is proposed by \cite{thrampoulidis2015regularized} and allows for accurate analysis of a large class of optimization problems involving random matrices. The key idea is relating the original problem (Primary Optimization PO) to an Auxiliary Optimization (AO) problem. Given compact convex set $\Sc\in\R^p$, regularization parameter $\la>0$ and continuous convex function $\psi(\cdot):\R^p\rightarrow\R$, define the functions
\begin{align}
\Phi_\la(\X)&=\min_{\w\in\Sc}\max_{\tn{\ab}\leq \la}\ab^T\X\w+\psi(\w)=\min_{\w\in\Sc}\la\tn{\X\w}+\psi(\w)\label{primary}\\
\phi_\la(\g,\h)&=\min_{\w\in\Sc}\max_{\tn{\ab}\leq \la}\tn{\w}\g^T\ab-\tn{\ab}\h^T\w+\psi(\w)\\
&=\min_{\w\in\Sc}\la(\tn{\w}\tn{\g}-\h^T\w)_++\psi(\w)\label{auxil}
\end{align}
Suppose $\X\in\R^{n\times p},\g\in\R^n,\h\in\R^p\distas\Nn(0,1)$. Then, CGMT yields the following inequality for any $\mu\in\R,t>0$, 
\begin{align}
\Pro(|\Phi_\la(\X)-\mu|> t)\leq 2\Pro(|\phi_\la(\g,\h)-\mu|> t).\label{CGMT main}
\end{align}
In words, the right and left-hand side objectives are probabilistically equal.

\noindent {\bf{Relation to ridge regression:}} Observe that \eqref{primary} can easily be related to ridge regression which solves
\begin{align}
\min_{\bt}\Lc_\la(\bt)=\min_{\bt}\la \tn{\y-\X\bt}+\tn{\bt}.\label{ridge reg}
\end{align}
Recalling $\y=\X\btb+\sigma \z$ with $\z\distas\Nn(0,1)$ and applying the change of variable $\w=\btb-\bt$, we find
\[
\Lc_\la(\w)=\la \tn{[\X~\z]\begin{bmatrix}\w\\\sigma\end{bmatrix}}+\tn{\btb-\w}.
\]
Observe that $\X'=[\X~\z]\in\R^{n\times (p+1)}\distas\Nn(0,1)$ thus setting $\psi(\w)=\tn{\btb-\w}$, minimization over $\Lc(\w)$ has the exact same form as \eqref{primary} and CGMT is applicable with
\[
\Phi_\la(\X')=\min_{\w}\la \tn{\X'\begin{bmatrix}\w\\\sigma\end{bmatrix}}+\tn{\btb-\w}
\]
Covariance on the design matrix can be handled as well as described in Appendix \ref{derive formula}. 

\noindent {\bf{Over-parameterized Least-Squares:}} In Section \ref{train loss} we study over-parameterized least-squares which interpolates the training labels perfectly rather than using ridge regularization. Specifically, we solve the min Euclidian norm problem
\[
\arg\min_{\bt}\tn{\bt}\quad\text{subject to}\quad \y=\X\bt.
\]
Note that this corresponds to solving \eqref{ridge reg} with $\la\rightarrow\infty$. Using the same change of variable, we end up with the primary optimization
\[
\Phi_\infty(\X')=\min_{\w}\tn{\btb-\w}\quad\text{subject to}\quad \X'\begin{bmatrix}\w\\\sigma\end{bmatrix}=0.
\]
Unfortunately, CGMT framework for our scenario has two drawbacks due to technical issues. First, it only handles the regularization term and doesn't allow for random matrix constraints. Secondly, as mentioned earlierin \eqref{primary}, $\w$ has to lie on a compact set $\Sc$. Even $\w\in\R^p$ has to be addressed with care. We first have the following theorem which circumvents these issues. The following result is a corollary of Theorem \ref{thm closed} and allows for equality constraints on $\X$ and replaces compactness on $\Sc$ with closedness.
\begin{theorem}[CGMT with constraints] \label{lem cgmt constrained}Given a closed $\Sc$ and a continuous function $\psi$ satisfying $\lim_{\tn{\vb}\rightarrow\infty}\psi(\vb)=\infty$, define the PO and AO problems
\begin{align}
&\Phi_\infty(\X)=\min_{\w\in\Sc,\X\w=0}\psi(\w)\label{con prim}\\
&\phi_\infty(\g,\h)=\min_{\w\in\Sc,\tn{\w}\tn{\g}\leq \h^T\w}\psi(\w)\label{con aux}.
\end{align}
Suppose $\X\in\R^{n\times p},\g\in\R^{n},\h\in\R^{p}\distas\Nn(0,1)$. Then, for any $t>0$ and $\mu\in\R$, we have that
\begin{itemize}
\item $\Pro(\Phi_\infty(\X) < t)\leq2\Pro(\phi_\infty(\g,\h)\leq t)$.
\item If $\Sc$ is convex, we additionally have $\Pro(\Phi_\infty(\X)> t)\leq2\Pro(\phi_\infty(\g,\h)\geq t)$.
\end{itemize}
\end{theorem}

\subsection{Using CGMT to Infer the Properties of the Solution}

In this section, we provide a discussion of how CGMT can be used to infer the properties of the solution of \eqref{primary} by studying the solution of \eqref{auxil}. This is already the topic of several interesting papers on random matrix theory and high-dimensional statistics \cite{thrampoulidis2015lasso,thrampoulidis2015regularized,thrampoulidis2018precise}. Below, we formalize the distributional similarity of the solution of the primary problem \eqref{primary} and auxiliary problem \eqref{auxil} in terms of subsets of $\R^p$ for which auxiliary solution concentrates on.

\begin{lemma} [AO solution to PO solution] \label{dist cont lem}Let $\X\in\R^{n\times p},\g\in\R^n,\h\in\R^p\distas\Nn(0,1)$. Suppose we have two loss functions $\Lc_{PO}(\w;\X)$ and $\Lc_{AO}(\w;\g,\h)$ as a function of $\w$\footnote{$\Lc(\w,\ab)$ can account for additional set constraints of type $\w\in\Cc$ by adding the indicator penalty $\max_{\la\geq 0}\la 1_{\w\not\in\Cc}$.}. Given a set $\Sc$, define the objectives 
\begin{align}
\Phi_{\Sc}(\X)=\min_{\w\in\Sc}\Lc_{PO}(\w;\X)\quad\text{and}\quad\phi_{\Sc}(\g,\h)=\min_{\w\in\Sc}\Lc_{AO}(\w;\g,\h).\label{general phi}
\end{align}
Suppose $\Phi$ and $\phi$ satisfies the following conditions for any closed set $\Sc$ and $t$
\begin{itemize}
\item $\Pro(\Phi_{\Sc}(\X) < t)\leq2\Pro(\phi_{\Sc}(\g,\h)\leq t)$.
\item Furthermore, if $\Sc$ is convex, $\Pro(\Phi_{\Sc}(\X) > t)\leq2\Pro(\phi_{\Sc}(\g,\h)\geq t)$.
\end{itemize}
Define the set of global minima $\Mc=\{\w\bgl \Lc(\w;\X)=\Phi(\X)\}$. For any closed set $\Sc$, we have that
\begin{align}
\Pro(\Mc\in\Sc^c)\geq 1-2\min_{t}(\Pro(\phi_{\R^p}(\g,\h)\geq t)+\Pro(\phi_{\Sc}(\g,\h)\leq t)).\label{min inside}
\end{align}
\end{lemma}
\begin{proof} Let $\w^*\in\Mc$. Suppose the events $\Phi_{\R^p}(\g,\h)\leq t$ and $\Phi_{\Sc}(\g,\h)> t$ hold. These two imply that $\w^*\not\in\Sc$ hence $\Mc\subseteq\Sc^c$. To proceed, for any choice of $t$
\begin{align}
\Pro(\Mc\in\Sc^c)&\geq \Pro(\{\Phi_{\R^p}(\g,\h)\leq t\}\cap \{\Phi_{\Sc}(\g,\h)> t\})\\
&\geq 1-\Pro(\Phi_{\R^p}(\g,\h)> t)-\Pro(\Phi_{\Sc}(\g,\h)\leq t)\\
&\geq 1-\Pro(\Phi_{\R^p}(\g,\h)> t)-\lim_{t'\rightarrow t^+}\Pro(\Phi_{\Sc}(\g,\h)<t')\\
&\geq 1-2\Pro(\phi_{\R^p}(\g,\h)\geq t)-2\lim_{t'\rightarrow t^+}\Pro(\phi_{\Sc}(\g,\h)\leq t').
\end{align}
Since this holds for all $t$ and cumulative distribution function is continuous, we get the advertised bound \eqref{min inside}.
\end{proof}

Note that assumptions of this lemma on the loss functions \eqref{general phi} holds for over-parameterized least-squares based on Theorem \ref{lem cgmt constrained}. In words, this lemma states that, if we can identify a set $\Sc$ such that $\Sc$-constrained auxiliary cost $\phi_{\Sc}(\g,\h)$ is larger than the unconstrained cost $\phi_{\R^p}(\g,\h)$, then, the solution of the primary problem provably lies on the complement $\Sc^c$.

Then, if we wish to prove the global minima $\Mc$ of the primary problem satisfies some property $\Pc$, the line of attack is as follows.
\begin{itemize}
\item Let $\Sc$ be the set of vectors not satisfying $\Pc$.
\item Show that $\phi_{\Sc}(\g,\h)>\phi_{\R^p}(\g,\h)$ with high probability.
\end{itemize}
In our application, we wish to argue that pruned auxiliary distribution $\Pi^M_s(\bta)$ achieves the same test error as the pruned primary solution $\Pi^M_s(\bth)$. Thus, the undesired set $\Sc$ can be defined as the set of vectors whose test error after pruning does not deviate much from the expected test error of pruned auxiliary solution $\bta=\btb-\w_{\text{aux}}$ i.e. (assuming $\bSi=\Iden$, the test error simplifies to Euclidian distance to the ground-truth $\btb$)
\[
\Sc=\{\w\bgl |\tn{\Pi^M_s(\btb-\w)-\btb}-\E\tn{\Pi^M_s(\bta)-\btb}|\leq \eps\},
\]
where $\eps>0$ is a knob which can approach $0$ asymptotically. Setting $\gamma=\E\tn{\Pi^M_s(\bta)-\btb}$ and $f(\w)=\tn{\Pi^M_s(\btb-\w)-\btb}$, this can be simplified to
\[
\Sc=\{\w\bgl |f(\w)-\gamma|\leq \eps\},
\]
{\bf{Technical Challenge in Pruning Analysis:}} Here, the technical challenge is analyzing the auxiliary problem over $\Sc$ which is a highly non-convex set due to the hard-thresholding operator. Even showing the concentration of the auxiliary error $\tn{\Pi^M_s(\bta)-\btb}$ around its expectation $\gamma$ is not trivial. If $f(\w)$ is a Lipschitz function of $\w$, $\Sc$ is a more manageable set and it is typically relatively easy to show that its elements are bounded away from zero (in Euclidian norm). Once $\Sc$ is bounded away from zero, what remains is showing optimization over $\Sc$ leads to a strictly larger loss since the set doesn't include global minima in it with high probability. We again remark that using soft-thresholding based pruning would be an easier path to theoretical guarantees and fully formalizing the pruning formulas as the soft-thresholding operator $\text{shrink}_T(x)=\max(x-T,0)$ is Lipschitz.

Finally, the next subsection derives the auxiliary distribution of Definition \ref{aux_def} by solving the auxiliary problem associated with the over-parameterized least-squares.

\subsection{Deriving the Auxiliary Distribution (Definition \ref{aux_def})}\label{derive formula}
\subsubsection{Over-parameterized Least-Squares with Diagonal Covariance}
Let us first set the exact problem we are analyzing. Let $\X\in\R^{n\times p}$ have zero-mean and normally distributed rows with a diagonal covariance matrix $\bSi=\E[\x\x^T]$. Given ground-truth vector $\bt$ and labels $\y=\X\bt+\sigma \z$, we consider the least-squares problem subject to the minimum Euclidian norm constraint (as $\kappa=p/n>1$) given by
\begin{align}
\min_{\bt'}\tn{\bt'}\quad\text{subject to}\quad \y=\X\bt'.\label{this problem}
\end{align}
Next subsection \ref{next subsec} will adapt the analysis of this subsection to obtain Def.~\ref{aux_def}. Using change of variable $\bt'=\bt-\w$, optimization problem \eqref{this problem} leads to
\begin{align}
\Phi(\X)=\min_{\w} \tn{\bt-\w}\quad\text{subject to}\quad \X\w+\sigma \z=0.
\end{align}
Write $\X=\Xb\sqrt{\bSi}$ where $\Xb\distas\Nn(0,1)$. Noticing $\tn{\X\w+\sigma \z}=\tn{\Xb\sqrt{\bSi}\w+\sigma \z}$, and recalling the constrained CGMT forms \eqref{con prim} and \eqref{con aux}, the auxiliary problem takes the form
\begin{align}
\phi(\g,\h)=\min_{\w} \tn{\bt-\w}\quad\text{subject to}\quad \tn{\g}\tn{\sqrt{\bSi}\w~{\sigma}}\leq \h^T\sqrt{\bSi}\w+\sigma h.\label{aux prob}
\end{align}
where $\g\sim\Nn(0,\Iden_n)$, $\h\sim\Nn(0,\Iden_p)$, $h\sim \Nn(0,1)$. Set $\bar{\h}=\h/\sqrt{p}$. Letting $p\rightarrow\infty$ and setting $\kappa=p/n$ a constant, observe that $h/\tn{\g}\rightarrow 0$, $h/\tn{\g}=\sqrt{\kappa}\bar{h}$, and we have pointwise convergence (over $\w$) to the problem
\begin{align}
\phi(\g,\h)=\min_{\w} \tn{\bt-\w}\quad\text{subject to}\quad \tn{\sqrt{\bSi}\w~{\sigma}}\leq \sqrt{\kappa}\bar{\h}^T\sqrt{\bSi}\w.
\end{align}
Taking the squares of both sides, we find the equivalent optimization (which preserves the minima)
\begin{align}
\phi(\g,\h)=\min_{\w} \tn{\bt-\w}^2\quad\text{subject to}\quad \tn{\sqrt{\bSi}\w~{\sigma}}^2\leq \kappa(\bar{\h}^T\sqrt{\bSi}\w)^2,
\end{align}
Set $S(\w)=\bar{\h}^T\sqrt{\bSi}\w=\sum_{i=1}^p\bar{\h}_i\w_i\sqrt{\bSi_{i,i}}$. The optimization above can alternatively be written in the entrywise decomposed form
\begin{align}
\phi(\g,\h)=\min_{\w} \sum_{i=1}^p(\bt_i-\w_i)^2\quad\text{subject to}\quad \sigma^2+\sum_{i=1}^p\bSi_{i,i}\w_i^2\leq \kappa S(\w)^2.
\end{align}
Considering the Lagrangian form, we find
\begin{align}
\phi(\g,\h)=\min_{\w}\max_{\Xi\geq 0} \sum_{i=1}^p(\bt_i-\w_i)^2+\Xi[\sigma^2+\sum_{i=1}^p\bSi_{i,i}\w_i^2- \kappa S(\w)^2].\label{lagrange form}
\end{align}
We will decompose entries of $\w_i$ as a term dependent on $\bar{\h}_i$ and an independent bias term via 
\begin{align}
\w_i=\frac{\gamma_i}{\sqrt{\bSi_{i,i}}}\bar{\h}_i+\zeta_i\bt_i.\label{w form}
\end{align}
Also set the variable
\[
\Gamma=(\frac{1}{p}\sum_{i=1}^p\gamma_i)^2.
\]
Using Law of Large Numbers, we have 
\[
\lim_{p\rightarrow\infty}S(\w)=\E[\bar{\h}^T\sqrt{\bSi}\w]=\E[\sum_{i=1}^p\gamma_i\bar{\h}_i^2]=\sqrt{\Gamma},
\] 
and
\[
\lim_{p\rightarrow\infty}\sum_{i=1}^p(\bt_i-\w_i)^2=\E[\sum_{i=1}^p(\bt_i-\w_i)^2]=\sum_{i=1}^p (1-\zeta_i)^2\bt_i^2+\frac{\gamma_i^2}{p\bSi_{i,i}},
\]
and
\[
\lim_{p\rightarrow\infty}\sum_{i=1}^p\bSi_{i,i}\w_i^2=\E[\sum_{i=1}^p\bSi_{i,i}\w_i^2]=\sum_{i=1}^p\bSi_{i,i}\zeta_i^2\bt_i^2+\frac{\gamma_i^2}{p}.
\]
Thus, we rewrite the problem \eqref{lagrange form} as
\begin{align}
\phi(\g,\h)=\min_{\bz,\bg}\max_{\Xi\geq 0} \sum_{i=1}^p(1-\zeta_i)^2\bt_i^2+\frac{\gamma_i^2}{p\bSi_{i,i}}+\Xi[\sigma^2+\sum_{i=1}^p\bSi_{i,i}\zeta_i^2\bt_i^2+\frac{\gamma_i^2}{p}- \kappa \Gamma].
\end{align}
Differentiating with respect to $\gamma_i$ and $\zeta_i$, and recalling the definition of $\Gamma$, we obtain the equations
\begin{align}
\frac{\gamma_i}{p\bSi_{i,i}}+\Xi (\frac{\gamma_i}{p}-\frac{\kappa\sqrt{\Gamma}}{p})=0&\iff \gamma_i=\frac{\kappa\sqrt{\Gamma}}{1+(\Xi\bSi_{i,i})^{-1}}\label{gamma con}\\
(\zeta_i-1)\bt_i^2+\Xi \bSi_{i,i}\bt_i^2\zeta_i=0&\iff \zeta_i=\frac{1}{1+\Xi \bSi_{i,i}}.\label{zeta con}
\end{align}
Using the definition of $\Gamma$, we find that, $\Xi>0$ has to satisfy
\begin{align}
&\sqrt{\Gamma}=\frac{1}{p}\sum_{i=1}^p\gamma_i=\frac{1}{p}\sum_{i=1}^p\frac{\kappa\sqrt{\Gamma}}{1+(\Xi\bSi_{i,i})^{-1}}\iff \\
&1=\frac{\kappa}{p}\sum_{i=1}^p\frac{1}{1+(\Xi\bSi_{i,i})^{-1}}.\label{xi const}
\end{align}
Finally, since $\Xi>0$, we need to satisfy the complementary slackness i.e.~the term multiplying $\Xi$ has to be zero. This implies the equality
\begin{align}
\sigma^2+\sum_{i=1}^p \frac{\gamma_i^2}{p}+\bSi_{i,i}\zeta_i^2\bt_i^2=\kappa\Gamma.\label{comp slack}
\end{align}
In summary, following \eqref{w form}, we found that the solution to auxiliary problem \eqref{aux prob} has the form
\begin{align}
\w(\g,\h)=\bz\odot\bt+\bSi^{-1/2}\bg\odot\bar{\h},\nn
\end{align}
where $\bg,\bz\in\R^p$ are given by solving the following equations.
\begin{itemize}
\item $\Xi$ satisfies \eqref{xi const}. Note that there is a unique positive $\Xi$ solving this equation because when $\Xi=0$ right side is $p/n$ which is larger than one and the right side is strictly decreasing function of $\Xi$ thus mean-value theorem implies unique solution,
\item $\zeta_i$ satisfies \eqref{zeta con},
\item $\gamma_i$ satisfies \eqref{gamma con},
\item Finally $\Gamma$ satisfies \eqref{comp slack} which leads to (after substituting $\gamma_i$ definition)
\begin{align}
&\sigma^2+\sum_{i=1}^p \frac{\kappa^2\Gamma}{p(1+(\Xi\bSi_{i,i})^{-1})^2}+\bSi_{i,i}\zeta_i^2\bt_i^2=\kappa\Gamma\iff \sigma^2+\sum_{i=1}^p \bSi_{i,i}\zeta_i^2\bt_i^2=\kappa\Gamma(1-\frac{\kappa}{p}\sum_{i=1}^p(1+(\Xi\bSi_{i,i})^{-1})^{-2}),\nn
\end{align}
which yields
\begin{align}
\Gamma=\frac{\sigma^2+\sum_{i=1}^p \bSi_{i,i}\zeta_i^2\bt_i^2}{\kappa(1-\frac{\kappa}{p}\sum_{i=1}^p(1+(\Xi\bSi_{i,i})^{-1})^{-2})}.\label{Gam con}
\end{align}
Finally, the parameter distribution of the axuiliary problem is given by reversing the change of variable i.e.~
\begin{align}
\bta=\bt-\w(\g,\h)=(\onebb_p-\bz)\odot\bt-\bSi^{-1/2}\bg\odot\bar{\h},\label{aux form eq}
\end{align}
where $\bar{\h}\sim \Nn(0,\Iden_p/p)$.
\end{itemize}

\subsubsection{Obtaining the Auxiliary Distribution of Definition \ref{aux_def}}\label{next subsec}

The setup in Section \ref{sec linear} can be mapped to the previous section as follows.
\begin{itemize}
\item Feature covariance is $\bSi=\La^2$ for some diagonal matrix $\La$,
\item The ground-truth vector is $\bt=\La^{-1}\btb$ (as $\La^{-1}\btb$ is the population minima of $\Lc_\La$).
\end{itemize}
Plugging these into \eqref{zeta con}, \eqref{gamma con}, \eqref{xi const}, \eqref{Gam con} and finally the equation of the auxiliary solution \eqref{aux form eq} leads to Definition \ref{aux_def}. Specifically, the terms are stated in terms of $\La$ rather than $\bSi$ and we also remark that $\Gamma,\bta$ terms slightly differ due to the ground-truth vector mapping $\bt\leftrightarrow\La^{-1}\btb$.


\section{Larger Hessian Wins More}\label{sec winner}
This section proves Theorem \ref{lem win} and explains the tightness of its bounds. The following lemma is a standard result under smoothness (Lipschitz gradient) condition.
\begin{lemma} \label{bound grad loss}Suppose $\Lc$ has $L$-Lipschitz gradients and $\min_{\bt'}\Lc(\bt')\geq 0$. Then, we have that
\[
\tn{\nabla\Lc(\bt)}\leq  \sqrt{2L\Lc(\bt)}.
\]
\end{lemma}
\begin{proof} $L$-smoothness of the function implies
\[
\Lc(\ab)\leq \Lc(\bb)+\li\nabla \Lc(\bb),\ab-\bb\ri+\frac{L}{2}\tn{\ab-\bb}^2.
\]
Setting $\ab=\bb-\nabla \Lc(\bb)/L$, we find the desired result via
\[
\frac{\tn{\nabla \Lc(\bb)}^2}{2L}\leq \Lc(\bb)-\Lc(\ab)\leq \Lc(\bb)-\min_{\bt'}\Lc(\bt')\leq \Lc(\bb).
\]
\end{proof}

\subsection{Proof of Theorem \ref{lem win}}

\begin{proof} {\bf{Step 1: Proving \eqref{dist bound}}}: Our proof will be accomplished by carefully keeping track of the gradient descent dynamics for both parameters. Observe that if PPLS holds, then the full gradient satisfies PL condition with parameter $\mu=\sum_{i=1}^D\mu_i$ since
\[
\tn{\nabla\Lc(\bt)}^2\sum_{i=1}^D\tn{\frac{\pa }{\pa \bt_{\iii}}\Lc(\bt)}^2\geq 2\sum_{i=1}^D\mu_i\Lc(\bt)=2\mu\Lc(\bt).
\]
With this observation, the statement 
\begin{align}
\Lc(\bt_\tau)\leq (1-\eta\mu)^\tau \Lc(\bt_0)\label{standard polyak2}
\end{align} on linear convergence is standard knowledge on PL inequality. Denote the $i$th partial derivative via $\nabla_i\Lc(\bt_\tau)$. Using properties of Hessian and $L_i$-Lipschitzness of partial gradient with respect to $\bt_{\iii}$, note that overall function is $L=\sum_{i=1}^DL_i$-smooth using positive-semidefiniteness of Hessian and upper bounds on its block diagonals. Secondly using PL condition and Lemma \ref{bound grad loss}, we have that
\ysatt{sqrt}
\[
\frac{\tn{\nabla_i\Lc(\bt_\tau)}}{\tn{\nabla\Lc(\bt_\tau)}}\leq  \frac{\sqrt{L_i\Lc(\bt_\tau)}}{\sqrt{\mu\Lc(\bt_\tau)}}\leq \sqrt{L_i/\mu}.
\] 
Thus, we can write
\begin{align}
\tn{\bt_{\iii,\tau+1}-\bt_{\iii,0}}&\leq \tn{\bt_{\iii,\tau}-\bt_{\iii,0}}+\eta\tn{\nabla_i\Lc(\bt_\tau)}\\
&\leq \tn{\bt_{\iii,\tau}-\bt_{\iii,0}}+\eta\sqrt{L_i/\mu}\tn{\nabla\Lc(\bt_\tau)}.
\end{align}
For any $\eta\leq 1/L$, $L$-smoothness and PL condition guarantees
\begin{align}
&\Lc(\bt_{\tau+1})\leq \Lc(\bt_\tau)-\frac{\eta}{2}\tn{\nabla\Lc(\bt_\tau)}^2\implies\\
&\sqrt{\Lc(\bt_{\tau+1})}\leq \sqrt{\Lc(\bt_\tau)}-\frac{\eta}{4 \sqrt{\Lc(\bt_\tau)}}\tn{\nabla\Lc(\bt_\tau)}^2\\
&\sqrt{\Lc(\bt_{\tau+1})}\leq \sqrt{\Lc(\bt_\tau)}-\eta\sqrt{\mu/8}\tn{\nabla\Lc(\bt_\tau)}.
\end{align}
Define the Lyapunov function 
\[
\Vc_\tau=\sqrt{\Lc(\bt_\tau)}+\max_{1\leq i\leq D}C_i\tn{\bt_{\iii,\tau}-\bt_{\iii,0}}.
\]
We will find proper $C_i$'s such that $\Vc_\tau$ is non-increasing. Observe that
\[
\Vc_{\tau+1}-\Vc_\tau\leq C_i\eta\sqrt{L_i/\mu}\tn{\nabla\Lc(\bt_\tau)}-\eta \sqrt{\mu/8}\tn{\nabla\Lc(\bt_\tau)}\leq 0,
\]
when $C_i=\mu/\sqrt{8L_i}$. Thus we pick
\[
\Vc_\tau=\sqrt{\Lc(\bt_\tau)}+\max_{1\leq i\leq D}\frac{\mu}{\sqrt{8L_i}}\tn{\bt_{\iii,\tau}-\bt_{\iii,0}}.
\]
Since Lyapunov function is non-increasing, for all $\tau\geq0$, we are guaranteed to have
\[
\tn{\bt_{\iii,\tau}-\bt_{\iii,0}}^2\leq \frac{8L_i}{\mu^2}\Lc(\bt_0).
\]
What remains is upper bounding the contribution of $\bt_i$ to the objective function which is addressed next.

{\bf{Step 2: Proving \eqref{bad2}}}: 
Using the bound on $\Lc(\bt_\tau)$ and $L_i$-smoothness of the partial derivative with respect to $\bt_{\iii}$ and Lemma \ref{bound grad loss}, we find
\begin{align}
\tn{\frac{\pa}{\pa \bt_{\iii}}\Lc(\bt_\tau)}\leq \sqrt{2L_i (1-\eta\mu)^\tau \Lc(\bt_0)}.\label{state t}
\end{align}
At iteration $\tau$, define $\bt(t)=t\bt_\tau+(1-t)\bar{\bt}_\tau$ for $0\leq t\leq 1$. Observe that, via line integration, we can bound 
\begin{align}
|\Lc(\bt_\tau)-\Lc(\bar{\bt}_\tau)|\leq \sup_{0\leq t\leq 1}\tn{\frac{\pa}{\pa \bt_{\iii}} \Lc(\bt(t))}\tn{\bt_{\iii,\tau}-\bt_{\iii,0}}.\label{line int}
\end{align}
For the right-hand side, we use the earlier upper bound
\[
\tn{\bt_{\iii,\tau}-\bt_{\iii,0}}\leq  R.
\] 
Next, using \eqref{state t} and $L_i$-smoothness again, we also bound the gradient norm via
\begin{align}
\tn{\frac{\pa}{\pa \bt_{\iii}} \Lc(\bt(t))}&\leq \tn{\frac{\pa}{\pa \bt_{\iii}} \Lc(\bt_\tau)}+L_i\tn{\bt_\tau-\bar{\bt}_\tau}\\
&\leq RL_i+\sqrt{2L_i (1-\eta\mu)^\tau \Lc(\bt_0)}.
\end{align}
Recalling \eqref{line int} and substituting $R$, we find
\begin{align}
|\Lc(\bt_\tau)-\Lc(\bar{\bt}_\tau)|&\leq R^2L_i+R\sqrt{2L_i (1-\eta\mu)^\tau \Lc(\bt_0)}\\
&\leq \Lc(\bt_0)(8L_i^2/\mu^2+4(L_i/\mu)(1-\eta\mu)^{\tau/2})\\
&\leq \Lc(\bt_0)(8\kappa^2+4\kappa(1-\eta\mu)^{\tau/2}).
\end{align}
This yields our bound \eqref{bad2}.
\end{proof}

{\bf{Step 3: Proving \eqref{good2}}}: Throughout the remaining proof, let $\tilde{\bt}_\tau=[\bt_{\iib,0}~\bt_{\iii,\tau}]$ be the $\bt_{\iib}$-ablated vector which sets the entries $\bt_{\iib,\tau}$ of the $\tau$'th iterate to their initial state $\bt_{\iib,0}$. Similarly, let $\bar{\bt}_\tau=[\bt_{\iib,\tau}~\bt_{\iii,0}]$ be the $\bt_{\iii}$ ablated vector. By construction
\[
{\Icn_{\iib}(\bt_\tau,\bt_0)}=\Lc(\tilde{\bt}_\tau)-\Lc(\bt_\tau),\quad{\Icn_{\iii}(\bt_\tau,\bt_0)}=\Lc(\bar{\bt}_\tau)-\Lc(\bt_\tau).
\]
Set the distance parameter $R=\sqrt{8L_i\Lc(\bt_0)}/\mu\geq \tn{\bt_{\iii,\tau}}$ as a short hand notation. 

Applying Lemma \ref{bound grad loss} on $\bt_{\iii}$, for any $\bt(t)=t\bt_0+(1-t)\tilde{\bt}_\tau$, we have that
\[
\tn{\frac{\pa}{\pa \bt_{\iii}}\Lc(\bt(t))}\leq \sqrt{2L_i\Lc(\bt_0)}+RL_i.
\]
Consequently, using line integration bound and $\tn{\bt_0-\tilde{\bt}_\tau}=\tn{\bt_{\iii,\tau}}\leq R$, we get
\begin{align}
\Lc(\bt_0)-\Lc(\tilde{\bt}_\tau)&\leq |\Lc(\bt_0)-\Lc(\tilde{\bt}_\tau)|\leq R(\sqrt{2L_i\Lc(\bt_0)}+RL_i)\\
&\leq \Lc(\bt_0)(4L_i/\mu+8{L_i^2/\mu^2})\\
&\leq \Lc(\bt_0)(8\kappa^2+4\kappa).
\end{align}
Combining this with \eqref{standard polyak2}, we obtain the second bound \eqref{good2} via
\[
\frac{\Lc(\tilde{\bt}_\tau)-\Lc(\bt_\tau)}{\Lc(\bt_0)}\geq 1-8\kappa^2-4\kappa-(1-\eta\mu)^\tau.
\]

\subsection{Theorem \ref{lem win} is Tight} \label{subsec tight}
To demonstrate the tightness of Theorem \ref{lem win}, we consider an over-parameterized linear regression setup similar to \eqref{lin reg}. Consider $D$ feature sets $(\X_i)_{i=1}^D\in\R^{n\times p_i}$ with $p_i\geq n$ where we fit 
\begin{align}
\Lc(\bt)=\min_{\bt=(\bt_i)_{i=1}^D}0.5\tn{\y-\sum_{i=1}^D\X_i\bt_i}^2.\label{lin reg2}
\end{align}
Let $\Delta_i$ be the set of entries corresponding to $\bt_i$. PPLS holds over $\Delta_i$ with parameters $\mu_i=\sigma_{\min}(\X_i)^2$ and $L_i=\|\X_i\|^2$. The overall problem is a regression with the design matrix $\X=[\X_1~\dots~\X_D]\in\R^{n\times p}$ where $p=\sum_{i=1}^Dp_i$ and $\X$ satisfies the PL and smoothness bounds with $\mu=\sum_{i=1}^D\mu_i$ and $L=\sum_{i=1}^DL_i$. To proceed, we have the following proposition that proves the tightness of Theorem \ref{lem win} up to condition numbers $L_i/\mu_i$ and $L/\mu$. Specifically, this proposition provides bounds sharply complementing Theorem \ref{lem win} by using the properties of the minimum $\ell_2$ norm solution to \eqref{lin reg2} which is the solution gradient descent converges to starting from zero initialization.
\begin{proposition}\label{tight} Let $\bt^\st=(\bt^\st_i)_{i=1}^D$ be the solution found by gradient descent on the loss \eqref{lin reg2} starting from an initialization $\bt_0$ (with learning rate $\eta\leq 1/L$). Set $\tilde{\kappa}=\mu_i/L$. Then, $\bt^\st_i$ satisfies the following bounds
\begin{align}
&\tn{\bt_{\iii,\tau}-\bt_{\iii,0}}^2\geq 2{\tilde{\kappa}}\Lc(\bt_0)/L,\label{dist boundL}\\
&{\Icn_{\iii}(\bt_\tau,\bt_0)}/{\Lc(\bt_0)}\geq \tilde{\kappa}^2,\label{bad2L}\\
&{\Icn_{\iib}(\bt_\tau,\bt_0)}/{\Lc(\bt_0)}\leq 1-\tilde{\kappa}^2-2{\tilde{\kappa}}\quad\text{when}\quad n=1\label{good2L}.
\end{align}
\end{proposition}
In short, the bounds of this proposition perfectly complements the bounds of Theorem \ref{lem win} after accounting for the local condition number $L_i/\mu_i$ and global condition number $L/\mu$ associated with PL condition and smoothness. Specifically, we simply replace $\kappa=L_i/\mu$ with $\tilde{\kappa}=\mu_i/L$ and the converse bounds hold on $\tilde{\kappa}$ up to very small constants. We remark that \eqref{dist boundL} and \eqref{bad2L} holds generally whereas we show \eqref{good2L} for the special case of $n=1$. Note that $\frac{\kappa}{\tilde{\kappa}}=\frac{L_i}{\mu_i}\frac{L}{\mu}$ which is the multiplication of the local and global condition numbers. Thus, Theorem \ref{lem win} is tight up to these condition numbers and very small constants as claimed in the main body.

\begin{proof} Let $\bt^\dagger$ be the pseudo-inverse solution given by 
\[
\bt^\dagger=\X^\dagger\y=\X^T(\X\X^T)^{-1}\tilde{\y}
\]
where $\tilde{\y}=\y-\y_0$ and $\y_0=\X\bt_0$. Gradient descent solution on linear least-squares converges to minimum Euclidian distance solution given by $\bt^\st=\bt_0+\bt^\dagger$. Observe that $\Lc(\bt_0)=\tn{\tilde{\y}}^2/2$ and 
\[
\tn{\bt^\dagger_\ixi}^2=\tn{\X_i^T(\X\X^T)^{-1}\tilde{\y}}^2\geq \frac{\sigma_{\min}(\X_i)^2}{\|\X\|^4}\tn{\tilde{\y}^2}\geq \tilde{\kappa}(2\Lc(\bt_0))/L.
\]
This proves the first statement of \eqref{dist boundL}. To show the second statement, note that at $\bt^\st$, the loss is equal to zero thus, the $\ix_i$ pruned vector $\bt^p=\bt^\st_{\bar{\ix}_i}+\bt_{0,{\ix}_i}$ achieves a loss of 
\[
\Lc(\bt^p)=0.5\tn{\X_i\bt^\dagger_{\ix_i}}^2\geq 0.5\tn{\X_i\X_i^T(\X\X^T)^{-1}\tilde{\y}}^2\geq \tilde{\kappa}^2\Lc(\bt_0),
\]
yielding \eqref{bad2L}. Finally, we look at the $\bar{\ix}_i$ pruned vector $\bt^p=\bt^\st_{\ix_i}+\bt_{0,\bar{\ix}_i}$. In this case, we wish to show that loss function $\Lc(\bt^p)$ is upper bounded. We have that
\begin{align}
2\Lc(\bt^p)&=\tn{\tilde{\y}-\X_i\bt^\dagger_{\ix_i}}^2\\
&=\tn{\tilde{\y}-\X_i\X_i^T(\X\X^T)^{-1}\tilde{\y}}^2\\
&=\tn{(\Iden_n-\X_i\X_i^T(\X\X^T)^{-1})\tilde{\y}}^2.
\end{align}
To proceed, note that, when $n=1$, $\Iden_n\succeq \X_i\X_i^T(\X\X^T)^{-1}\succeq (\mu_i/L)\Iden_n=\tilde{\kappa}\Iden_n$. Consequently, 
\[
2\Lc(\bt^p)\leq (1-\tilde{\kappa})^2\tn{\tilde{\y}}^2=2(1-\tilde{\kappa})^2\Lc(\bt_0),
\]
concluding the proof of \eqref{good2L}.
\end{proof}

\section{Proofs of Lemmas \ref{cov importance} and \ref{simple hessian gradient}}\label{app lemma proof}
\subsection{Proof of Lemma \ref{cov importance}}
\begin{proof} The least-squares loss evaluated at a point $\bt$ with design covariance $\bSi$ is given by
	\[
	\E[(y-\x^T\bt)^2]=\E[y^2]-2\bb^T\bt+\bt^T\bSi\bt.
	\]
	We first show that HI and NI is invariant to the scaling $\La$ regardless of the covariance $\bSi$. Observe that the covariance of $\x^\La$ is $\bSi^{\La}=\La\bSi\La$ and $\btb^\La=\La^{-1}\btb$. Consequently, we find that
	\[
	\Ic_\ix^H(\btb^{\La})=\sum_{i\in \ix}\bSi^\La_{i,i}(\btb^\La_i)^2=\sum_{i\in \ix}\La_{i,i}^2\bSi_{i,i}(\La^{-1}\btb)_i^2=\Ic_\ix^H(\btb).
	\]
	For NI, observing $\bb^\La=\La\bb$ and accounting for the $\La$ cancellations, we similarly have
	\begin{align}
	\Lc_\La(\btb^\La_{\bar{\ix}})-\Lc_\La(\btb^\La)&=[-2{\bb^\La}^T\btb^\La_{\bar{\ix}}+(\btb^\La_{\bar{\ix}})^T\bSi^{\La}\btb^\La_{\bar{\ix}}]-[-2{\bb^\La}^T\btb^\La+(\btb^\La)^T\bSi^{\La}\btb^\La]\\
	&=[-2{\bb}^T\btb_{\bar{\ix}}+(\btb_{\bar{\ix}})^T\bSi\btb_{\bar{\ix}}]-[-2\bb^T\btb+(\btb)^T\bSi\btb]\\
	&=2{\bb}^T\btb_{\ix}+\btb_{\bar{\ix}}^T\bSi\btb_{\bar{\ix}}-\btb^T\bSi\btb,
	\end{align}
	which is independent of $\La$. To proceed, we focus on diagonal covariance matrix $\bSi$. For HI/NI, we only need to show the result for $\La=\Iden_p$ and establish $\Icn_\ix(\btb) = \Ic_\ix^H(\btb)$. We can then apply the $\La$ invariance result  above. The least-squares loss for diagonal covariance evaluated at a point $\bt$ can be written as
	\[
	\E[(y-\x^T\bt)^2]=\E[y^2]-2\sum_{i=1}^p\bb_i\bt_i+\bSi_{i,i}\bt_i^2.
	\]
	Note that $\btb_i=\bb_i/\bSi_{i,i}$. Thus, recalling the definition of $\btb_{\bar{\ix}}$, we establish the desired HI equal to NI bound as follows
	\[
	\Icn_\ix(\btb)=\Lc(\btb_{\bar{\ix}})-\Lc(\btb)=2\sum_{i\in \ix}\bb_i\btb_i-\bSi_{i,i}{\btb_i}^2=\sum_{i\in \ix}\frac{\bb_i^2}{\bSi_{i,i}}=\sum_{i\in \ix}{\bSi_{i,i}}{{\btb_i}^2}=\Ic^H_\ix(\btb).
	\]
	
	Finally, magnitude-based importance with diagonal covariance is simply given by $\Ic_\ix^M(\btb^{\La})=\sum_{i\in\ix}(\btb^\La_i)^2=\sum_{i\in\ix}(\La^{-1}\btb)_i^2=\sum_{i\in \ix}\La_{i,i}^{-2}{\btb_i}^2$.
\end{proof}

\subsection{Proof of Lemma \ref{simple hessian gradient}}

\begin{proof} The first statement on MI immediately follows from the definition of MI and the construction of $\bt^\la$. For the remaining statements, we analyze the gradient and Hessian as a function of $\la$. Since Hessian and gradient are linear, we can focus on a single example $(\x,y)$. To prevent notational confusion, let us denote the point of evaluation by $(\W_0,\V_0)$ and the input/output layer variables by $(\W,\V)$. Thus, suppose $\bt^1=(\W_0,\V_0)$ and $\bt^\la=(\la \W_0,\la^{-1}\V_0)$. Use shorthand $f=f_{\bt^\la}(\x)$ which is invariant to $\la$. Let $L_1=\nabla_f\ell(y,f)\in\R^{K}$ and $L_2=\nabla^2_f\ell(y,f)\in\R^{K\times K}$. Let input layer have $p_I=m\times d$ parameters and output layer has $p_O=K\times m$ parameters. Also denote the partial first and second order derivatives of input layer w.r.t.~prediction $f$ via $F_{1,\la}^{\W}=\nabla_{\W}f_{\bt^\la}(\x)\in\R^{p_I\times K}$ and $F_{2,\la}^{\W}=\nabla^2_{\W}f_{\bt^\la}(\x)\in\R^{p_I\times p_I\times K}$. Similarly denote the partial derivatives of the output layer by $F_{1,\la}^{\V}=\nabla_{\V}f_{\bt^\la}(\x)\in\R^{p_O\times K}$ and $F_{2,\la}^{\V}=\nabla^2_{\V}f_{\bt^\la}(\x)\in\R^{p_O\times p_O\times K}$. First, focusing on gradient (of the vectorized input/output layers), we have the size $p_I,p_O$ partial gradients
\begin{align}
&\nabla_{\W}\ell(y,f_{\bt^\la}(\x))=F_{1,\la}^{\W}L_1\\
&\nabla_{\V}\ell(y,f_{\bt^\la}(\x))=F_{1,\la}^{\V}L_1.
\end{align}
Let $\mu(\cdot)$ be the step function which will correspond to the activation pattern. To proceed, observe that $\act{\la\W_0\x}=\la\act{\W_0\x}$ and $\mu(\la\W_0\x)=\mu(\W_0\x)$.
\begin{align}
F_{1,\la}^{\W}=\nabla_{\W=\la\W_0}(\la^{-1}\V_0\act{\W\x})&=\nabla_{\W=\W_0}(\la^{-1}\V_0\act{\W\x})\\
&=\la^{-1}\nabla_{\W=\W_0}(\V_0\act{\W\x})\\
&=\la^{-1}F_{1,1}^{\W}\\
F_{1,\la}^{\V}=\nabla_{\V=\la^{-1}\V_0}(\V\act{\la\W_0\x})&=\nabla_{\V=\V_0}(\V\act{\la\W_0\x})\\
&=\la\nabla_{\V=\V_0}(\V\act{\W_0\x})\\
&=\la F_{1,1}^{\V}.
\end{align}
which are the advertised results on gradient.

We next proceed with the Hessian analysis and show similar behavior to gradient. Let us use $\bOt$ to denote the tensor-vector multiplication which multiplies an $a\times b\times c$ tensor with a size $c$ vector along the third mode to find an $a\times b$ matrix. Note that
\begin{align}
\nabla^2_{\W}\ell(y,f_{\bt^\la}(\x))&=F_{2,\la}^{\W}\bOt L_1+F_{1,\la}^{\W}L_2{F_{1,\la}^{\W}}^T\\
&=F_{2,\la}^{\W}\bOt L_1+\la^{-2}F_{1,1}^{\W}L_2{F_{1,1}^{\W}}^T\\
\nabla^2_{\V}\ell(y,f_{\bt^\la}(\x))&=F_{2,\la}^{\V}\bOt L_1+F_{1,\la}^{\V} L_2{F_{1,\la}^{\V}}^T\\
&=F_{2,\la}^{\V}\bOt L_1+\la^2F_{1,1}^{\V} L_2{F_{1,1}^{\V}}^T.
\end{align}
Thus, to conclude with the proof of \eqref{gradient change}, we will show that $F_{2,\la}^{\V}=0$ and $F_{2,\la}^{\W}=\la^{-2}F_{2,1}^{\W}$. For the input layer, we use the fact that second derivative of ReLU is the Dirac $\delta$ function which satisfies $\delta(x/C)=C\delta(x)$ for $C>0$. Thus, we find
\begin{align}
F_{2,\la}^{\W}&=\nabla^2_{\W=\la\W_0}(\la^{-1}\V_0\act{\W\x})\\
&=\la^{-1}\nabla^2_{\W=\la\W_0}(\V_0\act{\W\x})\\
&=\la^{-1}\nabla^2_{\W=\W_0}(\la^{-1}\V_0\act{\W\x})\\
&=\la^{-2}\nabla^2_{\W=\W_0}(\V_0\act{\W\x})=\la^{-2}F_{2,1}^{\W}.
\end{align}
Similarly, $f_{\bt^\la}$ is a linear function of the output layer thus
\[
F_{2,\la}^{\V}=\nabla^2_{\V=\la^{-1}\V_0}(\V\act{\la\W_0\x})=0.
\]
This proves that Hessian exhibits the advertised behavior \eqref{gradient change}. Finally, \eqref{constant HI} follows from the fact that the diagonal entries of the Hessian of the input layer decays with $\la^2$ whereas its entries grow with $\la$ so that HI remains unchanged (and similar story for the output layer). 
\end{proof}

\begin{figure}[t!]
	\begin{subfigure}{2.1in}
		\centering
		\begin{tikzpicture}[scale=1,every node/.style={scale=1}]
		\node at (0,0) {\includegraphics[scale=0.36]{figs/retrain_naive_acc.pdf}};
		\node at (-2.8,0) [rotate=90,scale=1.]{Test accuracy};
		\node at (0,-1.8) [scale=1.]{Fraction of non-zero};
		\node at (-1.1,-0.7) [scale=.6] {$\V=0$};
		\node at (1.1,-0.7) [scale=.6] {$\W=0$};
		\draw [->,>=stealth] (-1.2,-0.8) -- (-1.45,-0.95);
		\draw [->,>=stealth] (1.1,-0.8) -- (0.85,-0.95);
		\end{tikzpicture}\caption{\small{Test accuracy using MP and standard pruning. }}\label{fig:compare_purining_C10_1_app}
	\end{subfigure}~\begin{subfigure}{2.1in}
		\centering
		\begin{tikzpicture}[scale=1,every node/.style={scale=1}]
		\node at (0,0) {\includegraphics[scale=0.36]{figs/retrain_jacobian_acc.pdf}};
		\node at (-2.8,0) [rotate=90,scale=1.]{Test accuracy};
		\node at (0,-1.8) [scale=1.]{Fraction of non-zero};
		
		\end{tikzpicture}\caption{\small{Test accuracy using HP and standard pruning.}}\label{fig:compare_purining_C10_2_app} 
	\end{subfigure}~\begin{subfigure}{2.1in}
		\centering
		\begin{tikzpicture}[scale=1,every node/.style={scale=1}]
		\node at (0,0) {\includegraphics[scale=0.36]{figs/width_init.pdf}};
		\node at (-2.8,0) [rotate=90,scale=1.]{Remaining output layer };
		\node at (2.8,0) [rotate=90,scale=1.]{Remaining input layer };
		\node at (0,-1.8) [scale=1.]{$\lambda$};
		\end{tikzpicture}\caption{\small{Remaining nonzeros for $\V,\W$ with standard pruning}}\label{fig:compare_purining_C10_3_app} 
	\end{subfigure}
\begin{subfigure}{2.1in}
	\centering
	\begin{tikzpicture}[scale=1,every node/.style={scale=1}]
	\node at (0,0) {\includegraphics[scale=0.36]{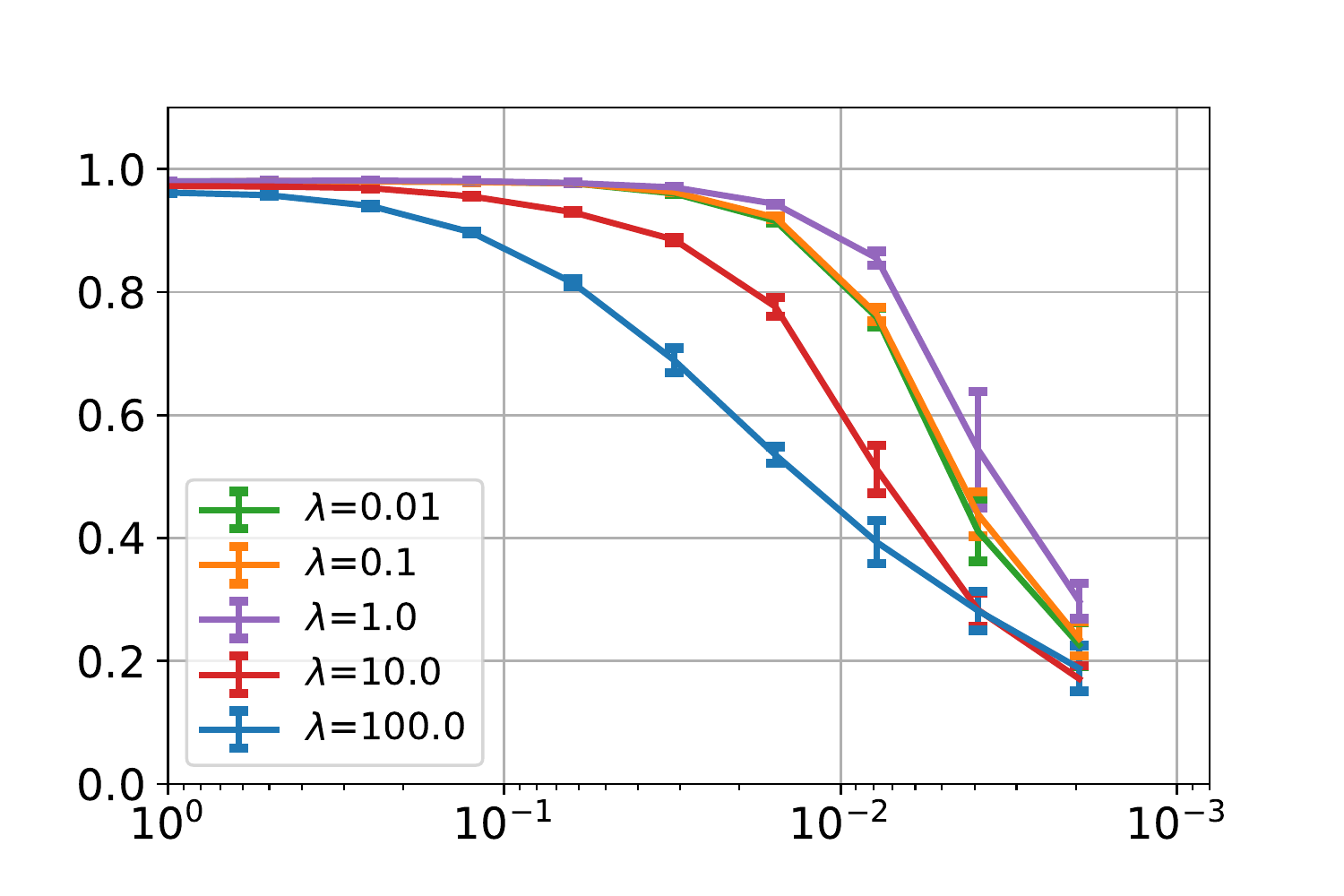}};
	\node at (-2.8,0) [rotate=90,scale=1.]{Test accuracy};
	\node at (0,-1.8) [scale=1.]{Fraction of non-zero};
	\end{tikzpicture}\caption{\small{Test accuracy using MP and layer-wise pruning. }}\label{fig:compare_purining_C10_4_app}
\end{subfigure}~\begin{subfigure}{2.1in}
	\centering
	\begin{tikzpicture}[scale=1,every node/.style={scale=1}]
	\node at (0,0) {\includegraphics[scale=0.36]{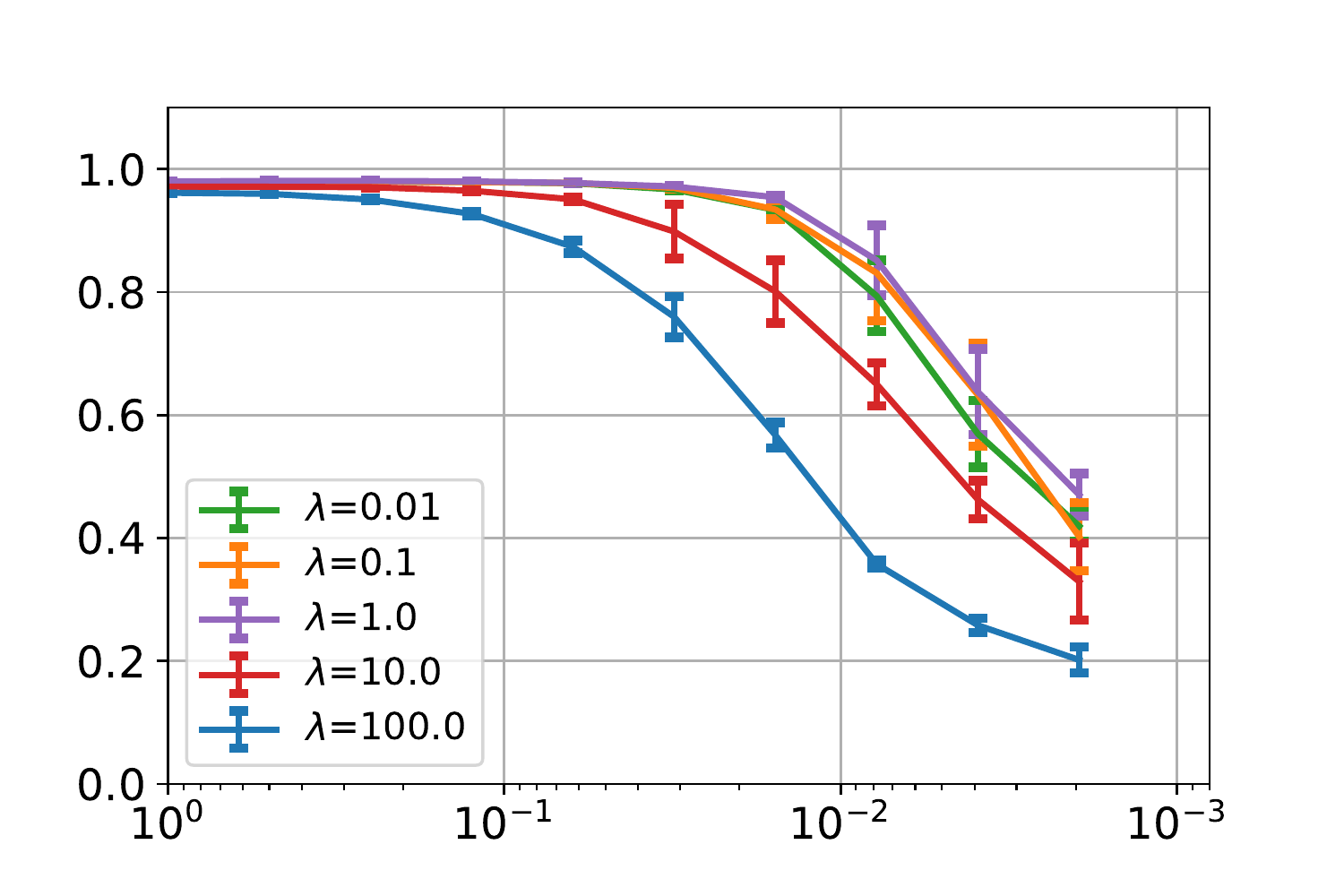}};
	\node at (-2.8,0) [rotate=90,scale=1.]{Test accuracy};
	\node at (0,-1.8) [scale=1.]{Fraction of non-zero};
	
	\end{tikzpicture}\caption{\small{Test accuracy using HP and layer-wise pruning.}}\label{fig:compare_purining_C10_5_app} 
\end{subfigure}~\begin{subfigure}{2.1in}
	\centering
	\begin{tikzpicture}[scale=1,every node/.style={scale=1}]
	\node at (0,0) {\includegraphics[scale=0.36]{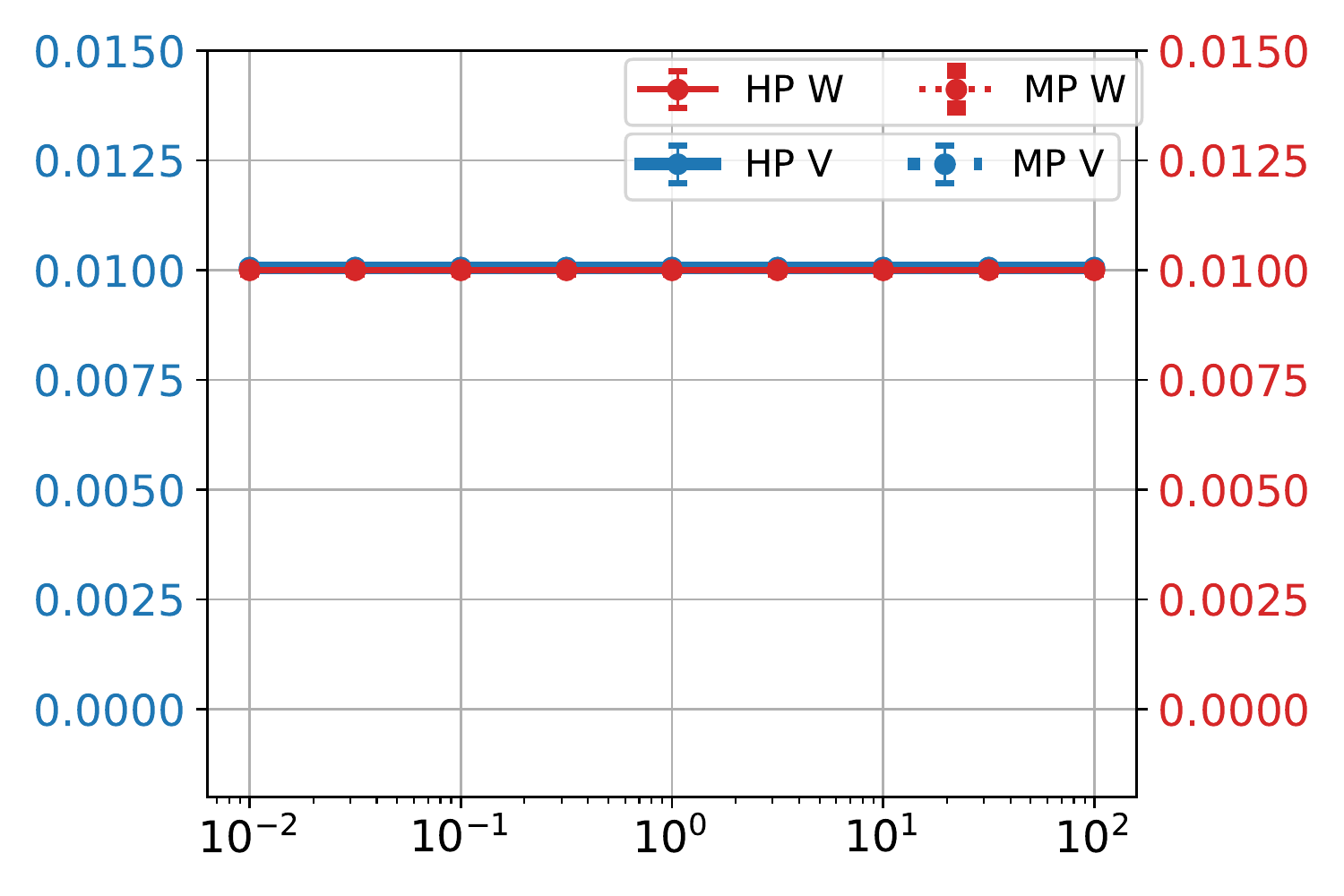}};
	\node at (-2.8,0) [rotate=90,scale=1.]{Remaining output layer };
	\node at (2.8,0) [rotate=90,scale=1.]{Remaining input layer };
	\node at (0,-1.8) [scale=1.]{$\lambda$};
	\end{tikzpicture}\caption{\small{Remaining nonzeros for $\V,\W$ with layer-wise pruning}}\label{fig:compare_purining_C10_6_app} 
\end{subfigure}\caption{\small{This figure compares layer-wise pruning with standard pruning. Figures (a), (b) and (c) are the same figures in Fig.~\ref{fig:compare_pruning_C10}. We use the same setup on (d), (e) and (f) except we use layer-wise pruning instead of standard pruning. Compared to standard MP, layer-wise MP dose not suffer from a full layer dying but the performance is worse when $\la=1$. Moreover, standard HP outperforms layer-wise HP except in the extremely sparse regime (nonzero $\leq 0.2\%$). In this regime, both approaches result in lackluster accuracy.}}\label{fig:compare_pruning_C10_app}\vspace{-12pt}

\end{figure}

\section{Further Experiments and Comparison to Layer-wise Pruning}\label{app layerwise}
In Section~\ref{neural sec} we used standard network pruning which prunes the whole set of weights to a certain sparsity level regardless of which layer they belong. We observed that MP can completely prune a layer when we apply very large or small $\la$-scaling in Fig~\ref{fig:compare_purining_C10_1}. We also showed HP significantly mitigates this problem as it is inherently invariant to $\la$. Layer-wise pruning prunes the exact same fraction of the parameters in each layer individually and it is an alternative way to avoid the {\em{layer death}} problem. Thus, in this section, we compare standard pruning with layer-wise pruning and display the results in Fig.~\ref{fig:compare_pruning_C10_app}.  Fig.~\ref{fig:compare_purining_C10_1_app} and \ref{fig:compare_purining_C10_4_app} show that layer-wise MP mitigates the layer death problem under $\la$-scaling because it keeps the same fraction of nonzero parameters in each layer. However when $\la=1$ the performance of layer-wise MP is worse than standard MP. Note that there is nothing special about $\la=1$ except the fact that input dimension (784) and number of hidden nodes (1024) are close to each other and He initialization results in input and output weights of similar magnitudes.

Fig. \ref{fig:compare_purining_C10_2_app} and \ref{fig:compare_purining_C10_5_app} compare standard HP with layer-wise HP showing that standard HP outperforms layer-wise HP except when the network is extremely spares (fraction of nonzero $\leq 0.2\%$). Our explanation for this behavior is as follows: The weights of certain layers (specifically output layer) are more important, in average, than others (specifically input layer). The standard HP fully takes this into account by jointly pruning the complete set of weights based on importance. In Figure \ref{fig:compare_purining_C10_3_app} it can be seen that, for 1\% sparsity target, standard HP keeps around 50\% of the output layer whereas layer-wise HP keeps exactly the target value 1\% (Fig \ref{fig:compare_purining_C10_6_app}). However the fact that standard HP favors the output layer weights results in input layer getting overly pruned in the extremely sparse regime and in this regime layer-wise pruning has a slight edge. However both methods lead to lackluster accuracy ($\sim 40\%$ accuracy on MNIST) in this regime, thus for practical purposes, it is plausible to say standard HP is better than or equal to layer-wise in all sparsities.


\section{Relaxing Conditions on Convex Gaussian Min-Max Theorem}\label{app further}
The following lemma replaces the compactness constrained with the closedness in CGMT. It also applies to problems with random equality constraints (which is of interest for over-parameterized least-squares) besides regularized form.
\begin{theorem} [Flexible CGMT]\label{thm closed} Let $\psi$ be a function obeying $\lim_{\tn{\w}\rightarrow\infty}\psi(\w)=\infty$. Given a closed set $\Sc$, define
\begin{align}
\Phi_\la(\X)&=\min_{\w\in\Sc}\la\tn{\X\w}+\psi(\w)\\
\phi_\la(\g,\h)&=\min_{\w\in\Sc}\la(\tn{\w}\tn{\g}-\h^T\w)_++\psi(\w),
\end{align}
and
\begin{align}
&\Phi_\infty(\X)=\min_{\w\in\Sc,\X\w=0}\psi(\w)\\
&\phi_\infty(\g,\h)=\min_{\w\in\Sc,\tn{\w}\tn{\g}\leq \h^T\w}\psi(\w).
\end{align}
For all $\la\in[0,\infty)\cup\{\infty\}$, we have that
\begin{itemize}
\item $\Pro(\Phi_\la(\X)<t)\leq2\Pro(\phi_\la(\X)\leq t)$.
\item If $\Sc$ is additionally convex, we additionally have that $\Pro(\Phi_\la(\X)>t)\leq2\Pro(\phi_\la(\X)\geq t)$. Combining with the first statement, this implies that for any $\mu,t>0$
\[
\Pro(|\Phi_\la(\X)-\mu|>t)\leq2\Pro(|\phi_\la(\X)-\mu|\geq t)
\]
\end{itemize}
\end{theorem}
\begin{proof} As an application of Theorem 3 of \cite{thrampoulidis2015regularized} and Lemma \ref{lem convex} and Lemma \ref{lem general constraint}, these two statements hold for a compact $\Sc$ and compact convex $\Sc$ respectively. We remark that Theorem 3 of \cite{thrampoulidis2015regularized} does not explicitly state $\Pro(\Phi_\la(\X)>t)\leq2\Pro(\phi_\la(\X)\geq t)$. However it is explicitly stated in the proof of this theorem (see Proof of Eq (13) in pg 22). Our goal is extending the proof to closed sets rather than compact. To achieve this, we consider a sequence of problems with the sets 
\[
\Sc_r=\{\x\bgl \tn{\x}\leq r\}\cap \Sc.
\]
$\Sc_r$ is compact thus the advertised inequalities hold for $\Sc_r$. The remaining argument is showing pointwise convergence and applying the Dominated Convergence Theorem as in the proofs of Lemma \ref{lem convex} and Lemma \ref{lem general constraint}. We will argue the result for $\la=\infty$. Finite $\la$ follows essentially the identical argument. Define $\Phi_{\la}^r(\X)=\min_{\w\in\Sc_r,\X\w=0}\psi(\w)$ and $\phi_\la^r(\g,\h)=\min_{\w\in\Sc_r,\tn{\w}\tn{\g}\leq \h^T\w}\psi(\w)$. Fix a matrix $\X$ and define the indicator $E_{\la}^r=1_{\Phi_{\la}^r(\X)< t}$. We claim that $\lim_{r\rightarrow\infty}E_{\la}^r=E_\la$. To see this consider the two cases: Case 1: If original problem is infeasible and $\Phi_{\la}(\X)=\infty$ then $\Phi_{\la}^r(\X)=\infty$ as well thus $\lim_{r\rightarrow\infty}E_{\la}^r=E_\la=0$. Case 2: $\Phi_{\la}(\X)$ is finite. By the divergence assumption on $\psi$, the set of optimal solutions $\w^*$ of the original problem achieving $\Phi_{\la}(\X)$ lie on a bounded $\ell_2$ set. Thus for sufficiently large $r$, $E_{\la}^r=E_\la$ (note that $\Phi_{\la}^r$ is a non-increasing function of $r$). To proceed, applying Dominated Convergence Theorem, this yields
\[
\lim_{r\rightarrow\infty}\E[E_{\la}^r]=\E[E_\la]\iff \Pro(\Phi_\la(\X)<t)=\lim_{r\rightarrow\infty}\Pro(\Phi_\la^r(\X)<t).
\]
Applying the same argument to $\phi_\la^r$ we obtain the desired bound
\begin{align}
\Pro(\Phi_\la(\X)<t)&=\lim_{r\rightarrow\infty}\Pro(\Phi_\la^r(\X)<t)\\
&\leq 2\lim_{r\rightarrow\infty}\Pro(\phi_\la^r(\g,\h)\leq t)\\
&\leq 2\Pro(\phi_\la(\g,\h)\leq t).
\end{align}
Repeating the identical/very similar arguments for the convex case and finite $\la$ (omitted for avoiding repetitions), we conclude the proof. Finally, the combination of upper and lower bounds yield the two sided bound by observing
\[
\Pro(|\Phi_\la(\X)-\mu|>t)=\Pro(\Phi_\la(\X)>\mu+t)+\Pro(\Phi_\la(\X)<\mu-t).
\]
\end{proof}
\subsection{Proof of Constrained CGMT}
\subsubsection{Proof for the convex case}
\begin{lemma} \label{lem convex}Given a convex and compact $\Sc$, define the PO and AO problems
\begin{align}
&\Phi_\infty(\X)=\min_{\w\in\Sc,\X\w=0}\psi(\w)\\
&\phi_\infty(\g,\h)=\min_{\w\in\Sc,\tn{\w}\tn{\g}\leq \h^T\w}\psi(\w).
\end{align}
Suppose $\X,\g,\h\distas\Nn(0,1)$. Then, we have that
\begin{align}
\Pro(\Phi_\infty(\X)> t)\leq2\Pro(\phi_\infty(\g,\h)\geq t).
\end{align}
\end{lemma}
\begin{proof}
Using convex-concavity of $\Lcp(\w,a)=a\tn{\X\w}+\psi(\w)$ we have that
\begin{align}
\Phi_\infty(\X)&=\min_{\w\in\Sc,\X\w=0}\psi(\w)\\
&=\max_{a\geq 0} \min_{\w\in\Sc}a\tn{\X\w}+\psi(\w)\\
&=\lim_{\la\rightarrow\infty}\max_{0\leq a\leq \la} \min_{\w\in\Sc}a\tn{\X\w}+\psi(\w)\\
&=\lim_{\la\rightarrow\infty}\min_{\w\in\Sc}\max_{0\leq a\leq \la} a\tn{\X\w}+\psi(\w)\\
&=\lim_{\la\rightarrow\infty} \Phi_\la(\X).
\end{align}
Note that if the problem is infeasible, both sides yield $\infty$. Similarly using convex-concavity of $\Lca(\w,a)=a(\tn{\w}\tn{\g}+\h^T\w)_++\psi(\w)$, we have
\[
\Phi_\infty(\g,\h)=\lim_{\la\rightarrow\infty} \phi_\la(\g,\h).
\]
Now that we connected the equality constrained problems $\Phi_\infty$ and $\phi_\infty$ to regularized problems, we proceed with establishing a probabilistic bound using CGMT. We remark that Theorem 3 of \cite{thrampoulidis2015regularized} does not explicitly state $\Pro(\Phi_\la(\X)>t)\leq2\Pro(\phi_\la(\X)\geq t)$. However it is explicitly stated in the proof of this theorem (see Proof of Eq (13) in pg 22). Define the indicator function $E_\la=1_{\Phi_\la(\X)> t}$. Observe that, for any choice of $\X$,
\[
\lim_{\la\rightarrow\infty}E_\la=\lim_{\la\rightarrow\infty}1_{\Phi_\la(\X)> t}=1_{\Phi_\infty(\X)> t}.
\]
Note that, if the problem is infeasible, then $\lim_{\la\rightarrow\infty}E_\la=E_\infty=1$. To proceed, we are in a position to apply Dominated Convergence Theorem to find
\begin{align}
&\lim_{\la \rightarrow\infty}\E[E_\la]=\E[E_\infty]\iff\Pro(\Phi_\infty(\X)> t)=\lim_{\la\rightarrow\infty}\Pro(\Phi_\la(\X)> t).
\end{align}
Applying the identical argument on $\phi_{\g,\h}$ to find $\Pro(\phi_\infty(\g,\h)\geq t)=\lim_{\la\rightarrow\infty}\Pro(\phi_\la(\g,\h)\geq t)$, we obtain the desired relation
\begin{align}
\Pro(\Phi_\infty(\X)> t)&=\lim_{\la\rightarrow\infty}\Pro(\Phi_\la(\X)> t)\\
&\leq 2\lim_{\la\rightarrow\infty}\Pro(\phi_\la(\g,\h)\geq t)\\
&= 2\Pro(\phi_\infty(\g,\h)\geq t).
\end{align}
\end{proof}
\subsubsection{Proof for the general case}
\begin{lemma} \label{lem general constraint}Given a compact set $\Sc$, define the PO and AO problems as in Lemma \ref{lem convex}. We have that
\begin{align}
\Pro(\Phi_\infty(\X)<t)\leq2\Pro(\phi_\infty(\g,\h)< t).
\end{align}
\end{lemma}
\begin{proof} The proof is similar to that of Lemma \ref{lem convex}. For a general compact set $\Sc$, application of Gordon's theorem yields the one-sided bound
\begin{align}
\Pro(\Phi_\la(\X)<t)\leq2\Pro(\phi_\la(\g,\h)\leq t).
\end{align}
To move from finite $\la$ to infinite, we make use of Lemma \ref{lem continuous limit}. Define the indicator function $E_\la=1_{\Phi_\la(\X)\leq t}$. Using Lemma \ref{lem continuous limit}, for any choice of $\X$, $\lim_{\la\rightarrow\infty}E_\la=\lim_{\la\rightarrow\infty}1_{\Phi_\la(\X)< t}=1_{\Phi_\infty(\X)< t}$. Note again that, if the problem is infeasible, then $\lim_{\la\rightarrow\infty}E_\la=E_\infty=0$. To proceed, we are in a position to apply Dominated Convergence Theorem to find
\begin{align}
&\lim_{\la \rightarrow\infty}\E[E_\la]=\E[E_\infty]\iff\Pro(\Phi_\infty(\X)< t)=\lim_{\la\rightarrow\infty}\Pro(\Phi_\la(\X)< t).
\end{align}
Applying the identical argument on $\phi_{\g,\h}$ to find $\Pro(\phi_\infty(\g,\h)\leq t)=\lim_{\la\rightarrow\infty}\Pro(\phi_\la(\g,\h)\leq t)$, we obtain the desired relation
\begin{align}
\Pro(\Phi_\infty(\X)< t)&=\lim_{\la\rightarrow\infty}\Pro(\Phi_\la(\X)< t)\\
&\leq 2\lim_{\la\rightarrow\infty}\Pro(\phi_\la(\g,\h)\leq t)\\
&= 2\Pro(\phi_\infty(\g,\h)\leq t).
\end{align}
\end{proof}
\begin{lemma}\label{lem continuous limit} Let $\Sc$ be a compact set and $\psi(\cdot)$ be a continuous function and $f(\w)$ be a non-negative continuous function. Then
\[
\lim_{\la\rightarrow\infty}\min_{\w\in\Sc}\la f(\w)+\psi(\w)=\min_{\w\in\Sc,f(\w)=0}\psi(\w)
\]
Thus, setting $f(\w)=\tn{\X\w}$ and $f(\w)=\tn{\w}\tn{\g}-\h^T\w$, we have that
\begin{align*}
&\lim_{\la\rightarrow\infty}\Phi_{\la}(\X)=\Phi_{\infty}(\X)\\
&\lim_{\la\rightarrow\infty}\phi_{\la}(\g,\h)=\phi_{\infty}(\g,\h).
\end{align*}
\end{lemma}
\begin{proof} 
Since $f$ is continuous, it has closed sub-level sets. Suppose $\{\w\in\Sc\bgl f(\w)=0\}= \emptyset$. Since $\Sc$ is compact, both sides are infinity and the equality holds. To proceed, we assume the problem is feasible. If $\min_{\w\in\Sc}\psi(\w)=\min_{\w\in\Sc,f(\w)=0}\psi(\w)$ again both sides are equal to $\min_{\w\in\Sc}\psi(\w)$ thus we assume  the right-hand side objective is strictly larger than $\min_{\w\in\Sc}\psi(\w)$. Define the sublevel sets $\Cc_{\alpha}=\Sc\cap \{\w\bgl f(\w)\leq \alpha\}$. 

Let $c_\la=\min_{\w\in\Sc}\la f(\w)+\psi(\w)$ and $c_\infty=\min_{\w\in\Sc,f(\w)=0}\psi(\w)$. Let $\w_{\la}=\arg\min_{\w\in\Sc}\la f(\w)+\psi(\w)$ and $\w_\infty=\arg\min_{\w\in\Sc,f(\w)=0}\psi(\w)$ be optimal solutions of regularized and constrained problems achieving $c_\la,c_\infty$ respectively. If the claim is wrong, then for some $\eps>0$ and all $\la>0$, $c_\la\leq c_\infty-\eps$. Since $f$ is nonnegative, this also implies that $\psi(\w_\la)\leq \psi(\w_\infty)-\eps$.

Since $\psi$ is a continuous function, $\psi$ uniformly converges on $\Sc$. Uniform convergence implies that for any $\eps>0$, there exists $\delta>0$ such that for all pairs $\tn{\w-\vb}<\delta$, we have $|\psi(\w)-\psi(\vb)|< \eps$. Conversely, if $|\psi(\w)-\psi(\vb)|\geq \eps$, we have that $\tn{\w-\vb}\geq\delta$. In our context, this means that, for all $\la\geq 0$
\[
\text{dist}(\w_\la,\Cc_0)\geq \delta.
\]
Set $\Gamma=\psi(\w_\infty)-\min_{\w\in\Sc}\psi(\w)>0$. For any $\la\geq 0$, $\la f(\w_\la)\leq \Gamma\implies f(\w_\la)\leq \Gamma/\la\implies \w_\la\in \Cc_{\Gamma/\la}$. This implies that for any choice of $\alpha>0$ (via $\alpha\leftrightarrow \Gamma/\la$), $\Cc_\alpha$ contains points that are $\delta$ away from $\Cc_0$. Note that $\Cc_\alpha$ is a non-decreasing sequence of sets (i.e.~$\Cc_{\alpha_1}\subseteq\Cc_{\alpha_2}$ whenever $\alpha_1\leq \alpha_2$). Via Bolzano–Weierstrass theorem $(\w_\la)_{\la\geq \Gamma}$ contains a convergent subsequence. Index this subsequence by $(\w_{\la_i})_{i=1}^\infty$ and suppose $\bar{\w}=\lim_{i\rightarrow\infty} \w_{\la_i}$. Clearly $\text{dist}(\bar{\w},\Cc_0)\geq \delta$ as distance is a continuous function. Note that $\bar{\w}\in \Cc_{\alpha}$ for any $\alpha>0$ since $\Cc_{\alpha}$ is non-decreasing and compact thus $\Cc_{\alpha}$ contains all the elements of $(\w_{\la_i})_{i=1}^\infty$ after a certain point including its limit. Finally, define $\bar{\Cc}=\lim_{\alpha\rightarrow0^+} \Cc_\alpha=\bigcap_{\alpha>0}\Cc_\alpha$. Clearly $\bar{\w}\in\bar{\Cc}$. This means that $\bar{\Cc}$ contains the element $\bar{\w}$ which is not inside $\Cc_0$. Finally, this leads to contradiction since $\bar{\Cc}\subseteq\Cc_0$. Specifically, if $\bar{\w}\in\bar{\Cc}$, then this implies 
\[
f(\bar{\w})\leq \alpha~\text{for all}~\alpha> 0\implies f(\bar{\w})=0\implies \bar{\w}\in\Cc_0.
\]
This concludes the proof.
\end{proof}



\end{document}